\documentclass{article}

\usepackage{
	amsmath,
	amsthm,
	amssymb,
	wrapfig,
	preview,
	cases,
	mathtools,
	thmtools,
	array,
	bbm,
	bm,
	subcaption,
	makecell,
	esvect,
	mathrsfs,
	breqn,
	booktabs,
	upgreek,
	changepage,
	dsfont,
	fixme,
	listings,
	multirow,
	xargs,
	xstring,
	xparse,
	multicol,
	graphicx,
	float,
	color,
	algpseudocode,
	enumitem,
	indentfirst,
	accents,
	ifthen,
	wasysym,
	tikz,
	pgf,
	pgfplots,
	lmodern,
	titlesec,
    colortbl
	}

\usepackage[utf8]{inputenc}
\usepackage[linesnumbered,boxruled,lined,noend]{algorithm2e}
\SetAlFnt{\footnotesize}
\SetKw{Break}{break}
\SetKwComment{tcp}{$\triangleright$\ }{}

\usepackage[hidelinks]{hyperref}

\usepackage{natbib}
\setcitestyle{authoryear,round,citesep={;},aysep={,},yysep={;}}
\usepackage{paralist}
\renewenvironment{enumerate}[1]{\begin{compactenum}#1}{\end{compactenum}}

\usetikzlibrary{shapes,arrows,trees,automata,positioning,matrix,decorations.pathreplacing,calc}
\pgfplotsset{compat=1.16}

\usepackage[capitalize,nameinlink]{cleveref}
\usepackage{crossreftools}
\hypersetup{%
    bookmarksnumbered, bookmarksopen=true, bookmarksopenlevel=1,%
}

\crefname{subsection}{Subsection}{Subsections}
\crefname{lemma}{Lemma}{Lemmas}
\crefname{corollary}{Corollary}{Corollaries}
\crefname{theorem}{Theorem}{Theorems}
\crefname{assumption}{Assumption}{Assumptions}

\declaretheorem[name=Theorem]{theorem}

\declaretheorem[sibling=theorem,name=Lemma]{lemma}
\declaretheorem[sibling=theorem,name=Proposition]{proposition}
\declaretheorem[sibling=theorem,name=Corollary]{corollary}

\declaretheorem[sibling=theorem,name=Definition]{definition}

\declaretheorem[name=Assumption, numbered=no]{assumption*}

\declaretheorem[qed=$\triangleleft$,sibling=theorem,name=Remark]{remark}

\numberwithin{equation}{section}
\numberwithin{theorem}{section}
\numberwithin{lemma}{section}
\numberwithin{proposition}{section}
\numberwithin{corollary}{section}
\numberwithin{conjecture}{section}
\numberwithin{definition}{section}
\numberwithin{assumption}{section}
\numberwithin{example}{section}
\numberwithin{remark}{section}
\numberwithin{note}{section}
\numberwithin{fact}{section}

\makeatletter
\renewcommand{\maketag@@@}[1]{\hbox{\m@th\normalsize\normalfont#1}}%
\makeatother

\makeatletter
\let\reftagform@=\tagform@
\def\tagform@#1{\maketag@@@{\ignorespaces\textcolor{gray}{(#1)}\unskip\@@italiccorr}}
\renewcommand{\eqref}[1]{\textup{\reftagform@{\ref{#1}}}}
\makeatother

\renewcommand{\AA}{\mathbb{A}}

\newcommand{\EE}{\mathbb{E}}
\newcommand{\FF}{\mathbb{F}}

\newcommand{\KK}{\mathbb{K}}
\newcommand{\LL}{\mathbb{L}}

\newcommand{\NN}{\mathbb{N}}

\newcommand{\PP}{\mathbb{P}}

\newcommand{\RR}{\mathbb{R}}

\newcommand{\Aa}{\mathcal{A}}

\newcommand{\Cc}{\mathcal{C}}

\newcommand{\Ff}{\mathcal{F}}

\newcommand{\Ii}{\mathcal{I}}

\newcommand{\Ll}{\mathcal{L}}

\newcommand{\Nn}{\mathcal{N}}

\newcommand{\Ss}{\mathcal{S}}
\newcommand{\Tt}{\mathcal{T}}

\newcommand{\Zz}{\mathcal{Z}}

\def\[#1\]{\begin{equation}\begin{aligned}#1\end{aligned}\end{equation}}
\def\*[#1\]{\begin{equation*}\begin{aligned}#1\end{aligned}\end{equation*}}
\def\s*[#1\s]{\small\begin{align*}#1\end{align*}\normalsize}

\newcommand{\lcrx}[4][{-1}]{ 
	\IfEq{#1}{-1}{\left #2 {{{{#3}}}} \right #4}{
   	\IfEq{#1}{0}{#2 {{{{#3}}}} #4}{
	\IfEq{#1}{1}{\bigl #2 {{{{#3}}}} \bigr #4}{
	\IfEq{#1}{2}{\Bigl #2 {{{{#3}}}} \Bigr #4}{
	\IfEq{#1}{3}{\biggl #2 {{{{#3}}}} \biggr #4}{
	\IfEq{#1}{4}{\Biggl #2 {{{{#3}}}} \Biggr #4}{
    \GenericWarning{"4th argument to lcrx must be -1, 0, 1, 2, 3, or 4"}
    }}}}}}} % specify size with {-1,...4} as optional argument

\newcommand{\suchthat}{\;\ifnum\currentgrouptype=16 \middle\fi|\;} % Vertical bar
\newcommand{\Ind}{\mathds 1} % indicator function
\def\multiset#1#2{\ensuremath{\left(\kern-.3em\left(\genfrac{}{}{0pt}{}{#1}{#2}\right)\kern-.3em\right)}}

\DeclareMathOperator*{\argmin}{\arg\min} % argmin
\DeclareMathOperator*{\argmax}{\arg\max} % argmax
\DeclareMathOperator*{\esssup}{ess\sup} % essential supremum
\DeclareMathOperator*{\newlim}{\mathrm{lim}\vphantom{\mathrm{infsup}}}
\DeclareMathOperator*{\newmin}{\mathrm{min}\vphantom{\mathrm{infsup}}}
\DeclareMathOperator*{\newmax}{\mathrm{max}\vphantom{\mathrm{infsup}}}
\DeclareMathOperator*{\newinf}{\mathrm{inf}\vphantom{\mathrm{infsup}}}
\DeclareMathOperator*{\newsup}{\mathrm{sup}\vphantom{\mathrm{infsup}}}
\renewcommand{\lim}{\newlim}
\renewcommand{\min}{\newmin}
\renewcommand{\max}{\newmax}
\renewcommand{\inf}{\newinf}
\renewcommand{\sup}{\newsup}

\renewcommand{\vec}[1]{\boldsymbol{#1}}

\newcommand{\dee}{\mathrm{d}} % for integrals \int f(x) \dee x
\newcommand{\grad}{\nabla} % gradient
\DeclareDocumentCommand{\virtualDiff}{m m G{} G{}}{\frac{#1^{#4} \kern 1pt #3}{#1 \kern 1pt #2^{#4}}}
\DeclareDocumentCommand{\pdiff}{m G{} G{}}{\virtualDiff{\partial}{#1}{#2}{#3}} % Partial derivative \pdiff{denominator}{optional numerator}{optional exponant}
\DeclareDocumentCommand{\diff}{m G{} G{}}{\virtualDiff{\text d}{#1}{#2}{#3}} % Derivative

\newcommand{\bernoullidist}{\mathrm{Ber}}

\newcommand{\sbra}[2][{-1}]{\lcrx[#1] [ {#2} ] }

\newcommand{\Nats}{\NN}

\newcommand{\Reals}{\RR}

\newcommand{\PosReals}{\Reals_+}

\newcommand{\range}[2][{1}]{
	\IfEq{#1}{1}{\sbra{#2}}{\sbra{#2}_{#1}}}
\newcommand{\rangeO}[2][{0}]{
	\IfEq{#1}{0}{\sbra{#2}_0}{\sbra{#2}_{#1}}}

\newcommand{\RL}{RL} % reinforcement learning
\newcommand{\ANN}{ANN} % artificial neural network
\newcommand{\CDF}{CDF} % cumulative distribution function
\newcommand{\PnL}{P\&L} % profit and losses
\newcommand{\DP}{DP} % dynamic programming
\newcommand{\MLP}{MLP} % multilayer perceptron
\newcommand{\VECM}{VECM} % vector error correction model

\newcommand{\riskmeas}{\rho}
\newcommand{\robust}{\varrho}
\newcommand{\wass}[1]{d_{#1}}

\newcommand{\adv}{\phi}
\DeclareMathOperator{\CVaR}{CVaR}
\DeclareMathOperator{\VaR}{VaR}

\newcommand{\Lpspace}{\Zz}
\newcommand{\rv}{Z}
\newcommand{\rvdum}{W}

\newcommand{\price}{S}
\newcommand{\statespace}{\Ss}
\newcommand{\state}{s}
\newcommand{\statedum}{\state'}

\newcommand{\actionspace}{\Aa}
\newcommand{\action}{a}

\newcommand{\costspace}{\Cc}
\newcommand{\costfunc}{c}

\newcommand{\timespace}{\Tt}
\newcommand{\timelength}{T}
\newcommand{\timeidx}{t}

\newcommand{\policy}{\pi}
\newcommand{\policyparams}{\theta}
\newcommand{\valuefunc}{V}

\newcommand{\qfunc}{Q}
\newcommand{\qparams}{\psi}
\newcommand{\ffunc}{F}
\newcommand{\fparams}{\vartheta}
\newcommand{\mufunc}{\mu}
\newcommand{\muparams}{\chi}

\newcommand{\lrate}{\eta}

\newcommand{\estim}{\mathfrak{a}}
\newcommand{\estimsupp}{\AA}
\newcommand{\score}{S}
\newcommand{\rvsupp}{\KK}

\renewcommand{\grad}[1]{\nabla_{#1}}
\definecolor{mgreen}{rgb}{0,0.455,0.247}
\definecolor{mblue}{rgb}{0.098,0.18,0.357}
\definecolor{mred}{rgb}{0.902,0.4157,0.0196}
\definecolor{mgrey}{rgb}{0.90196,0.90,0.90}
\definecolor{mpurple}{HTML}{9467bd}
\definecolor{mblack}{rgb}{0,0,0}

\usepackage{todonotes}
\usepackage{arxiv}
\usepackage[T1]{fontenc}
\usepackage{setspace}
\title{Robust Reinforcement Learning\\with Dynamic Distortion Risk Measures}
\author{Anthony Coache\thanks{AC acknowledges support from the Fonds de recherche du Québec -- Nature et technologies and Ontario Graduate Scholarship programs.} \\
	Department of Mathematics\\
         Imperial College London\\
	\href{mailto:a.coache@imperial.ac.uk}{a.coache@imperial.ac.uk} \\
	\url{https://anthonycoache.ca/} \\
	\And
	Sebastian Jaimungal\thanks{SJ acknowledges support from the Natural Sciences and Engineering Research Council of Canada (NSERC) for partially funding this work through grants RGPIN-2018-05705 and RGPIN-2024-04317.} \\
	Department of Statistical Sciences\\
	University of Toronto\\
        \&
        \\
        Oxford-Man Institute of Quantitative Finance
        \\
	\href{mailto:sebastian.jaimungal@utoronto.ca}{sebastian.jaimungal@utoronto.ca} \\
	\url{http://sebastian.statistics.utoronto.ca/}}

\renewcommand{\headeright}{}
\renewcommand{\undertitle}{}
\renewcommand{\shorttitle}{Robust RL with Dynamic Distortion Risk}

\hypersetup{
pdftitle={Robust Reinforcement Learning with Dynamic Distortion Risk Measures},
pdfsubject={},
pdfauthor={Anthony ~Coache, Sebastian ~Jaimungal},
pdfkeywords={},
}

\begin{document}

\maketitle

\allowdisplaybreaks

\begin{abstract}
    In a reinforcement learning (\RL{}) setting, the agent's optimal strategy heavily depends on her risk preferences and the underlying model dynamics of the training environment. These two aspects influence the agent's ability to make well-informed and time-consistent decisions when facing testing environments. In this work, we devise a framework to solve robust risk-aware \RL{} problems where we simultaneously account for environmental uncertainty and risk with a class of dynamic robust distortion risk measures. Robustness is introduced by considering all models within a Wasserstein ball around a reference model. We estimate such dynamic robust risk measures using neural networks by making use of strictly consistent scoring functions, derive policy gradient formulae using the quantile representation of distortion risk measures, and construct an actor-critic algorithm to solve this class of robust risk-aware \RL{} problems. We demonstrate the performance of our algorithm on a portfolio allocation example.
\end{abstract}

% ARXIV

\section{Introduction}
\label{sec:introduction}
Reinforcement learning (\RL{}) is a model-agnostic framework for learning-based control.
In brief, an agent observes feedback from interactions with an environment, updates its current behavior according to its experience, and aims to discover the best possible actions based on a certain criterion.
\RL{} provides an appealing alternative  to model-based methods; indeed, with \RL{} various environmental assumptions come at a low computational cost, while solving analytically for the optimal policies may be difficult and are often intractable for complex models.
Advancements with neural network structures have paved the way to deep learning, which has shown a lot of success recently \citep[see e.g.][]{mnih2015human,silver2018general,brown2019superhuman,berner2019dota}.

During the training phase of \RL{}, the agent attempts to discover the best possible strategy by interacting with a virtual representation of the environment, usually a simulation engine or historical data when the state process is exogenous, i.e. actions do not affect the distribution of the states.
The rationale behind this approach is that despite not interacting with the intended environment, training experience should reflect events similar to those likely to occur during the testing phase.
Uncertainty in the real-world environment, however, may result in algorithms optimized on training models to perform poorly during testing.
Therefore, it is crucial to consider robustifying the agent's actions, that is to account for inherent environmental uncertainty in sequential decision making problems.

There exist several distances to construct uncertainty sets and robustify optimization problems, such as the Kantorovich distance \citep[see e.g.][]{pflug2007ambiguity}, Kullback-Leibler divergence \citep[see e.g.][]{glasserman2014robust}, distances originating from mass transportation \citep[see e.g.][]{blanchet2019quantifying}, or the supremum over all risk induced by Bayesian mixture probability measures \citep{cuchiero2022risk}.
In the literature, robust Markov decision processes (MDPs) in a model-based setting were first initiated concurrently in \cite{nilim2005robust,iyengar2005robust} where uncertainties are uncoupled between states, a property known as rectangularity ambiguity. Many works extended this framework for computational improvements, including robust policy gradient methods \citep[see e.g.][and the references therein]{kumar2023policy,wang2023policy}. Despite the connections between risk-aware and robust MDPs \citep{osogami2012robustness,li2023rectangularity}, however, it remains unclear how to formally generalize those ideas to scalable model-free algorithms.
Some researchers have developed methodologies to account for model uncertainty in model-free \RL{} problems.
Among others, \cite{smirnova2019distributionally} propose a distributionally robust risk-neutral \RL{} algorithm, where the uncertainty set consists of all policies having a Kullback-Leibler divergence within a given epsilon of a reference action probability distribution, \cite{abdullah2019wasserstein} develop a robust risk-neutral \RL{} method, where the robustness is induced by considering all transition probabilities in a Wasserstein ball from a reference dynamics model, and \cite{clavier2022robust} give a robust distributional \RL{} algorithm constrained with a $\phi$-divergence on the transition probabilities.

Accounting for uncertainty in the testing environment is important, but ideally, risk must also simultaneously be accounted for.
Indeed, agents often want to follow a strategy that goes beyond ``on-average'' optimal performance, especially in mathematical finance applications with low-probability but high-cost outcomes.
Risk-aware \RL{}, or risk-sensitive \RL{}, aims to mitigate risk by replacing the expectation in the optimization problem with risk measures.
It also provides more flexibility than traditional risk-neutral approaches, because the agent may choose the measure of risk according to her own goals and risk preferences -- moreover, the agent may use risk measures to trade off risk and reward.

There are numerous proposals for optimizing static risk measures in sequential decision making problems.
One main issue with these works is that their proposed algorithms find optimal precommitment strategies, i.e., they result in time-inconsistent strategies.
In the more recent literature, many authors have attempted to overcome this issue with risk measures adapted to a dynamic setting.
Among others, \cite{bauerle2022markov} propose iterated coherent risk measures, where they both derive risk-aware dynamic programming (\DP{}) equations and provide policy iteration algorithms, \cite{ahmadi2021constrained} investigate bounded policy iteration algorithms for partially observable MDPs, \cite{kose2021risk} prove the convergence of temporal difference algorithms optimising dynamic Markov coherent risk measures, and \cite{cheng2025risk} derive a \DP{} principle for Kusuoka-type conditional risk mappings.
These works, however, require computing the value function for every possible state of the environment, limiting their applicability to problems with a small number of state-action pairs.

The articles closest in spirit to ours are those in \cite{jaimungal2022robust}, \cite{coache2023conditionally}, and \cite{bielecki2023risk}.
First, \cite{jaimungal2022robust} develop a deep \RL{} approach to solve a wide class of robust risk-aware \RL{} problems, where an agent minimizes the static worst-case rank dependent expected utility measure of risk of all random variables within a certain uncertainty set. It generalizes the approach from \cite{pesenti2023portfolio}, in which the authors aim to find an optimal strategy, whose terminal wealth is distributionally close to a benchmark's according to the $2$-Wasserstein distance, minimizing a static distortion risk measure of the terminal \PnL{} in a portfolio allocation application. \cite{wu2023robust} apply the approach to robustify path dependent option hedging.
Then, \cite{coache2023conditionally} design a deep \RL{} algorithm to solve time-consistent \RL{} problems where the agent optimizes dynamic spectral risk measures. It builds upon the work from \cite{coache2024reinforcement} by exploiting the conditional elicitability property of spectral risk measures to improve their estimation, and \cite{marzban2023deep} which focus on dynamic expectile risk measures. These ideas are also used in \cite{pesenti2025risk} for risk budgeting allocation with dynamic distortion risk measures.
Finally, \cite{bielecki2023risk} derive dynamic programming equations for risk-averse control problems with model uncertainty from a Bayesian perspective and partially observed costs. This approach simultaneously accounts for risk and model uncertainty, but requires finite state and action spaces.

To the best of our knowledge, this paper bridges the gaps between those works, as it simultaneously accounts for risk with dynamic risk measures and robustifies the actions against the uncertainty of the environment using the Wasserstein distance within the one-step conditional risk measures.
Our contributions may be summarized as follows:
(i) we consider robust risk-aware \RL{} problems with a class of dynamic robust distortion risk measures;
(ii) we analyze the worst-case distribution function of those dynamic robust distortion risk measures with uncertainty sets induced by the conditional Wasserstein distance via their quantile representation;
(iii) we derive a formula for computing the deterministic policy gradient using the Envelope theorem for saddle-point problems;
(iv) we devise a deep actor-critic style algorithm for solving those \RL{} problems, which optimizes a deterministic policy and estimates elicitable functionals using strictly consistent scoring functions;
and (v) we prove the existence of a neural network approximating the corresponding Q-function to any arbitrary accuracy. \cite{moresco2025uncertainty} proves (under certain technical assumptions) that the form of robustification that we utilize indeed leads to time-consistent optimal strategies, as well, they provide an even more general equivalence between time-consistency and robust dynamic risk measures. \cite{tam2025dimension} proves that our model uncertainty can be interpreted as an upper bound for distributionally robust problems with multivariate random variables under Lipschitz aggregation.

The remainder of this paper is structured as follows. \Cref{sec:risk} summarizes the fundamental concepts to formally define dynamic robust risk measures and their properties. We then introduce the class of \RL{} problems and explore their worst-case distributions for various dynamic robust distortion risk measures in \cref{sec:problem-setup}. \Cref{sec:algorithm} presents our developed actor-critic algorithm to solve those risk-aware \RL{} problems, and \cref{sec:UAT-results} provides some universal approximation theorem results for our approach. Finally, we illustrate the performance of our \RL{} methodology on a portfolio allocation application in \cref{sec:experiments} and explain our work's limitations and future extensions in \cref{sec:conclusion}.

\section{Risk Assessment}
\label{sec:risk}

In this section, we provide a brief overview of dynamic risk measures.
There exist several classes of static risk measures \citep[see e.g.][and the references therein]{follmer2016stochastic}, and various extensions to dynamic risk measures in the literature \citep[see e.g.][]{acciaio2011dynamic,bielecki2016dynamic}.
Here, we briefly summarize the work of \cite{ruszczynski2010risk}, which derives a recursive equation for dynamic risk measures using general principles, and present a class of one-step conditional robust distortion risk measures with their properties inline with the framework of \cite{moresco2025uncertainty}.

\subsection{Dynamic Risk}
\label{ssec:dynamic-risk}
While one may employ static risk measures in sequential decision making problems, doing so often leads to an optimal precommitment strategy, as static risk measures are not dynamically time-consistent risk measures -- we discuss the precise definition below.
In essence, an optimal strategy planned for a future state of the environment when optimizing a static risk measure may not be optimal anymore once the agent reaches this state.
Therefore, we must adapt risk assessment to a dynamic framework to properly monitor the flow of information.
For the remaining of this section, we follow the work of \cite{ruszczynski2010risk}.

Let $\timespace := \{0, \ldots, \timelength\}$ denote a sequence of periods, and define $\Lpspace_{\timeidx} := \Ll^{\infty}(\Omega, \Ff_{\timeidx}, \PP)$ as the space of bounded $\Ff_{\timeidx}$-measurable random variables.
We consider a filtration $\{\emptyset,\Omega\} =: \Ff_{0} \subseteq \Ff_{1} \subseteq \ldots \subseteq \Ff_{\timelength} \subseteq \Ff$ on a filtered probability space $(\Omega, \Ff, \{\Ff_{\timeidx}\}_{\timeidx\in\timespace}, \PP)$.
We also define $\Lpspace_{\timeidx_1,\timeidx_2} := \Lpspace_{\timeidx_1} \times \cdots \times \Lpspace_{\timeidx_2}$.
We assume that $\rv\in\Lpspace_{\timeidx+1}$, a $\Ff_{\timeidx+1}$-measurable random cost with support on $\rvsupp \in \bar{\Reals}$, has conditional cumulative distribution function (\CDF{}) and quantile function
\begin{equation*}
    F_{\timeidx}(z) := \PP\big(\rv \leq z \bigm| \Ff_{\timeidx} \big) \in [0,1] \quad \text{and} \quad \breve{F}_{\timeidx}(u) := \inf \big\{z\in\KK : F_{\timeidx}(z) \geq u \big\},
\end{equation*}
respectively. Furthermore, for the remaining of the paper, all inequalities between sequences of random variables are to be understood component-wise and in the almost sure sense.
\begin{definition}
    \label{def:dynamic-risk-measure}
    A \emph{dynamic risk measure} is a sequence of conditional risk measures $\{\riskmeas_{\timeidx,\timelength}\}_{\timeidx \in \timespace}$, where $\riskmeas_{\timeidx_1,\timeidx_2}$ is a map $\riskmeas_{\timeidx_1,\timeidx_2}: \Lpspace_{\timeidx_1,\timeidx_2} \rightarrow \Lpspace_{\timeidx_1}$ for any $\timeidx_1,\timeidx_2 \in \timespace$ such that $\timeidx_1 < \timeidx_2$.
\end{definition}
The mappings $\riskmeas_{\timeidx,\timelength}(\rv)$ may be interpreted as $\Ff_{\timeidx}$-measurable charges one would be willing to incur at time $\timeidx$ instead of the sequence of costs $\rv$.
We next enumerate some properties of various dynamic risk measures -- not all properties are required, but some, such as normalization, monotonicity, and cash additivity, are crucial in various settings.
\begin{definition}
    Let $\rv,\rvdum \in \Lpspace_{\timeidx,\timelength}$, and $\beta > 0$, where $\beta \in \Lpspace_{\timeidx}$. A dynamic risk measure $\{\riskmeas_{\timeidx,\timelength}\}_{\timeidx\in\timespace
    }$ is said to be the following if the statement holds for any $\timeidx\in\timespace$:
    \begin{enumerate}
        \setlength{\itemsep}{0pt}
        \item \emph{normalized} if $\riskmeas_{\timeidx,\timelength} (0,\ldots,0) = 0$;
        \item \emph{monotone} if $\rv \leq \rvdum$ implies $\riskmeas_{\timeidx,\timelength}(\rv) \leq \riskmeas_{\timeidx,\timelength}(\rvdum)$;
        \item \emph{cash additive} if $\riskmeas_{\timeidx,\timelength}(\rv_{\timeidx}, \rv_{\timeidx+1}, \ldots, \rv_{\timelength}) = \rv_{\timeidx} + \riskmeas_{\timeidx,\timelength}(0, \rv_{\timeidx+1}, \ldots, \rv_{\timelength})$;
        \item \emph{positive homogeneous} if $\riskmeas_{\timeidx,\timelength}(\beta \rv) = \beta\, \riskmeas_{\timeidx,\timelength}(\rv)$;
        \item \emph{subadditive} if $\riskmeas_{\timeidx,\timelength}(\rv + \rvdum) \leq \riskmeas_{\timeidx,\timelength}(\rv) + \riskmeas_{\timeidx,\timelength}(\rvdum)$;
        \item \emph{comonotonic additive} if $\riskmeas_{\timeidx,\timelength}(\rv+\rvdum) = \riskmeas_{\timeidx,\timelength}(\rv) + \riskmeas_{\timeidx,\timelength}(\rvdum)$ for comonotonic pairs $(\rv,\rvdum)$;
        \item \emph{coherent} if it is monotone, cash additive, positive homogeneous and subadditive.
    \end{enumerate}    
\end{definition}
A pivotal property of dynamic risk measures is their time-consistency, to ensure that risk assessments of future outcomes do not result in contradictions over time \citep[see e.g.][]{cheridito2006dynamic}.
\begin{definition}
    \label{def:time-consistency}
    A dynamic risk measure $\{\riskmeas_{\timeidx,\timelength}\}_{\timeidx \in \timespace}$ is said to be \emph{strongly time-consistent} iff for any sequence $\rv,\rvdum \in \Lpspace_{\timeidx_1,\timelength}$ and any $\timeidx_1,\timeidx_2 \in \timespace$ such that $0 \leq \timeidx_1 < \timeidx_2 \leq \timelength$,
    \begin{equation*}
        \riskmeas_{\timeidx_2,\timelength}(\rv_{\timeidx_2,\timelength}) \leq \riskmeas_{\timeidx_2,\timelength}(\rvdum_{\timeidx_2,\timelength})
        \quad \text{and} \quad
        \rv_{\timeidx_1,\timeidx_2-1} = \rvdum_{\timeidx_1,\timeidx_2-1}
    \end{equation*}
    implies that $\riskmeas_{\timeidx_1,\timelength}(\rv_{\timeidx_1,\timelength}) \leq \riskmeas_{\timeidx_1,\timelength}(\rvdum_{\timeidx_1,\timelength})$.
\end{definition}
\cref{def:time-consistency} may be interpreted as follows: if $\rv$ will be at least as good as $\rvdum$ tomorrow (in terms of the dynamic risk $\riskmeas_{\timeidx_2,\timelength}$) and they are identical today (between $\timeidx_1$ and $\timeidx_2$), then, all other things being equal, $\rv$ should not be worse than $\rvdum$ today (in terms of $\riskmeas_{\timeidx_1,\timelength}$).
A key result to derive a recursive relationship for strongly time-consistent dynamic risk measures is the following characterisation \citep[see Theorem 1 of][]{ruszczynski2010risk}.
\begin{proposition}
    \label{thm:time-consistency}
    Let $\{\riskmeas_{\timeidx,\timelength}\}_{\timeidx \in \timespace}$ be a dynamic risk measure satisfying the normalization, monotonicity and cash additivity properties.
    Then $\{\riskmeas_{\timeidx,\timelength}\}_{\timeidx \in \timespace}$ is time-consistent iff for any $0 \leq \timeidx_{1} \leq \timeidx_{2} \leq \timelength$ and $\rv \in \Lpspace_{0,\timelength}$, we have
    \begin{equation*}
        \riskmeas_{\timeidx_1,\timelength} (\rv_{\timeidx_1,\timelength}) = \riskmeas_{\timeidx_1,\timeidx_2} \Big(
        \rv_{\timeidx_1}, \ldots, \rv_{\timeidx_2 - 1},
        \riskmeas_{\timeidx_2, \timelength} \big(
        \rv_{\timeidx_2,\timelength}
        \big) \Big).
    \end{equation*}
\end{proposition}
As a consequence of \cref{thm:time-consistency}, for any $\timeidx \in \timespace$, we have the recursive relationship
\begin{equation}
    \riskmeas_{\timeidx,\timelength} (\rv_{\timeidx,\timelength}) = \rv_{\timeidx} +
    \riskmeas_{\timeidx} \bigg( \rv_{\timeidx+1} +
    \riskmeas_{\timeidx+1} \Big( \rv_{\timeidx+2} +
    \cdots +
    \riskmeas_{\timelength-2} \big( \rv_{\timelength-1} +
    \riskmeas_{\timelength-1} ( \rv_{\timelength} ) \big) \cdots \Big) \bigg), \label{eq:dynamic-risk}
\end{equation}
where the \emph{one-step conditional risk measures} $\riskmeas_{\timeidx}: \Lpspace_{\timeidx+1} \rightarrow \Lpspace_{\timeidx}$ satisfy $\riskmeas_{\timeidx} (\rv) = \riskmeas_{\timeidx, \timeidx+1} (0, \rv)$ for any $\rv\in\Lpspace_{\timeidx+1}$.
The one-step conditional risk measures may have stronger properties, e.g., be convex or coherent.
\cref{eq:dynamic-risk} provides a tractable expression to work with for deriving dynamic programming principles in \RL{} problems.

The following class of one-step conditional risk measures, introduced by \cite{yaari1987dual} in a static setting, subsumes many risk measures commonly used in the literature.
\begin{definition}
    \label{def:distortion-risk}
    Let $\nu:\Ff_{\timeidx+1}\times\Omega\rightarrow[0,1]$ be a regular conditional distribution of $\rv$ given $\Ff_{\timeidx}$, i.e. $\nu(\cdot,\omega)$ is a probability measure for any $\omega\in\Omega$ and $\nu(z,\cdot)$ is $\Ff_{\timeidx}$-measurable for any $z\in\Ff_{\timeidx+1}$.
    A \emph{one-step conditional distortion risk measure} $\riskmeas_{\timeidx}^{\gamma}:\Lpspace_{\timeidx+1} \rightarrow \Lpspace_{\timeidx}$ is defined as
    \begin{equation*}
        \riskmeas_{\timeidx}^{\gamma}(\rv)(\omega) = \EE \bigg[ \rv \; \gamma\Big(F_{\timeidx}(\rv)\Big) \biggm| \Ff_{\timeidx} \bigg](\omega) = \int_{0}^{1} \gamma(u) \breve{F}_{\timeidx}(u)(\omega) \dee u, \quad \text{for a.e. } \omega\in\Omega,
    \end{equation*}
    where $\gamma:[0,1] \rightarrow \PosReals$ satisfies $\int_{[0,1]} \gamma(u) \dee u = 1$.
\end{definition}
\begin{remark}
If the $\sigma$-algebra $\Ff_{\timeidx}$ is trivial, i.e. $\Ff_{\timeidx} = \{\emptyset,\Omega\}$, then $\riskmeas_{\timeidx}^{\gamma}(\rv)$ in \cref{def:distortion-risk} becomes the static distortion risk measure of a random cost $\rv$.
We suppress the dependence on $\omega$ for  one-step conditional risk measures in the sequel for readability.
\end{remark}

By the properties of Choquet integrals, such one-step conditional distortion risk measures are cash additive, monotone, positively homogeneous and comonotonic additive.
Such risk measures allow for risk-averse, risk-seeking, or partially risk-averse and risk-seeking attitudes by judiciously choosing the distortion function $\gamma$.

In practice, there is often uncertainty on distribution of the random costs $\rv$, and therefore we propose to robustify risk measures by using the worst-case risk of a random variable $\rv^{\adv}$ that lies within an $\epsilon$-Wasserstein ball of $\rv$.
Such robustification has desirable properties and is an important line of research in financial risk management.
Recently, \cite{moresco2025uncertainty} have investigated the inclusion of uncertainty sets within dynamic risk measures from a very general perspective. One of their results shows that, under suitable mild assumptions, considering uncertainty on the entire stochastic process is equivalent to considering one-step uncertainty sets. In our work, we use precisely this one-step uncertainty ball formulation. As well, \cite{tam2025dimension} have studied static robust optimization problems with uncertainty sets involving various multivariate Wasserstein distances. They prove that for Lipschitz aggregation functions, the uncertainty set of the aggregate random variable contains the aggregate of the multivariate uncertainty set. In our work, these uncertainty sets may be viewed as an upper bound on aggregate uncertainty.

\begin{definition}
    \label{def:wasserstein}
    Let $\langle f, g \rangle = \int_{[0,1]} f(u) g(u) \dee u$ be the $L^2$-inner product between two real functions $f,g$ on $[0,1]$ and $\Vert f \Vert^2 = \langle f, f \rangle$. The \emph{conditional Wasserstein distance} of order 2 between two random variables on $\Lpspace_{\timeidx+1}$ with conditional quantile functions denoted by $\breve{F}_{\timeidx},\breve{G}_{\timeidx}$ and distributions on the real line is given by
    \begin{equation*}
         \wass{\timeidx}^2(\breve{F}_{\timeidx},\breve{G}_{\timeidx})
        = \int_{0}^{1} \Big( \breve{F}_{\timeidx} (u) - \breve{G}_{\timeidx}(u) \Big)^{2} \, \dee u
        = \big\Vert \breve{F}_{\timeidx} - \breve{G}_{\timeidx} \big\Vert^2.
    \end{equation*}
\end{definition}
\begin{definition}
    \label{def:robust-risk}
    A \emph{robust} one-step conditional measure of risk $\riskmeas_{\timeidx}: \Lpspace_{\timeidx+1} \rightarrow \Lpspace_{\timeidx}$ under the uncertainty set $\varphi^{\epsilon}: \Lpspace_{\timeidx+1} \rightarrow 2^{\Lpspace_{\timeidx+1}}$ with tolerance $\epsilon \geq 0$ is defined as
    \begin{equation*}
        \robust_{\timeidx}^{\epsilon}(\rv) := \esssup_{\rv^{\adv} \in \varphi_{\rv}^{\epsilon}} \, \riskmeas_{\timeidx}(\rv^{\adv}).
    \end{equation*}
\end{definition}
In this paper, we consider the following two uncertainty sets induced by the conditional Wasserstein distance, where $F_{\adv,\timeidx}$ is the $\Ff_{\timeidx}$-conditional CDF of $\rv^{\adv}$ and $\breve{F}_{\adv,\timeidx}$ its quantile function:
\begin{subequations}
\begin{align}
	\vartheta_{\rv}^{\epsilon} &= \left\{ \rv^{\adv} \in \Lpspace_{\timeidx+1} \, : \, \begin{array}{l} \Vert \breve{F}_{\timeidx} - \breve{F}_{\adv,\timeidx} \Vert \leq \epsilon \end{array} \right\}, \quad \text{and} \label{eq:U-set-wass}
    \\[1em]
	\varsigma_{\rv}^{\epsilon} &= \left\{ \rv^{\adv} \in \Lpspace_{\timeidx+1} \, : \, \begin{array}{l}
    \Vert \breve{F}_{\timeidx} - \breve{F}_{\adv,\timeidx} \Vert \leq \epsilon, 
    \\[0.5em]
    \langle \breve{F}_{\timeidx}, 1 \rangle = \langle \breve{F}_{\adv,\timeidx}, 1 \rangle, 
    \\[0.5em]
    \Vert \breve{F}_{\timeidx} \Vert^2 = \Vert \breve{F}_{\adv,\timeidx} \Vert^2
  \end{array} \right\}. \label{eq:U-set-wass-mom}
\end{align}
\end{subequations}
\cref{eq:U-set-wass} contains all $\Ff_{\timeidx+1}$-measurable random variable that are distributionally close to $\rv$ wrt the conditional Wasserstein distance, while \cref{eq:U-set-wass-mom} additionally imposes they have the same first two moments.
These uncertainty sets expand as the tolerance $\epsilon$ increases, and one recovers the original risk measure $\riskmeas_{\timeidx}$ when $\epsilon=0$.
Here, $\epsilon$ in \cref{def:robust-risk} is directly driven by the agent's risk preferences. Rules of thumb for this tolerance $\epsilon$ are explored in \cref{sec:experiments}.

In what follows, we work with dynamic risk measures where each one-step conditional risk measure is a robust distortion risk measure, as described in \cref{def:distortion-risk,def:robust-risk}, with a tolerance $\epsilon_{\timeidx}\in\Ff_{\timeidx}$ and a piecewise constant distortion function $\gamma_{\timeidx}$:
\begin{equation*}
    \robust_{\timeidx}^{\epsilon_{\timeidx},\gamma_{\timeidx}}(\rv)
    := \esssup_{\rv^{\adv} \in \varphi_{\rv}^{\epsilon_{\timeidx}}} \, \EE \bigg[ \rv^{\adv} \; \gamma_{\timeidx} \Big(F_{\adv,\timeidx}(\rv^{\adv})\Big) \biggm| \Ff_{\timeidx} \bigg].
\end{equation*}
This class of risk measures incorporates uncertainty via robustification, allows risk-averse and risk-seeking behaviors with the distortion function, and is conditionally elicitable, because they may be written as a linear combination of $\CVaR$s.
We first explore uncertainty sets $\varphi_{\rv}^{\epsilon}$ satisfying \cref{eq:U-set-wass} in \cref{ssec:U-set-wass}, and then show in \cref{ssec:U-set-wass-mom} why including moment constraints as in \cref{eq:U-set-wass-mom} becomes essential for decision making problems.

\subsection{Elicitability}
\label{ssec:elicitability}
One desirable property for risk measures, which plays a crucial role in the algorithmic part of this work, is their elicitability.
\begin{definition}
    \label{def:elicitable}
    Let $\estimsupp \subseteq \bar{\Reals}^{k}$, for $k \geq 1$.
    A mapping $\riskmeas_{\timeidx}:\Lpspace_{\timeidx+1} \rightarrow \estimsupp$ is \emph{$k$-elicitable} \citep{gneiting2011making} iff there exists a strictly consistent scoring function $\score:\estimsupp \times \rvsupp \rightarrow \Reals$ for $\riskmeas$, that is, we have
    \begin{equation*}
        \EE_{\rv \sim F_t} \Big[ \score \big( \riskmeas_{\timeidx}(\rv), \rv \big) \Big]
        \leq
        \EE_{\rv \sim F_t} \Big[ \score \big( \estim, \rv \big) \Big],
    \end{equation*}
    for any $F_t$ and $\estim \in \estimsupp$, with equality when $\estim = \riskmeas_{\timeidx}(\rv)$.
\end{definition}
In view of \cref{def:elicitable}, a risk measure is elicitable if and only if there exists a scoring function $\score$ such that its estimate is the unique minimizer of the expected score, i.e.
\begin{equation*}
    \riskmeas_{\timeidx}(\rv) = \argmin_{\estim \in \estimsupp} \; \EE_{\rv \sim F_t} \Big[ \score(\estim,\rv) \Big].
\end{equation*}

To fix ideas, let us describe some strictly consistent scoring functions for well-known risk measures.
The mean is 1-elicitable and a strictly consistent scoring function must be of the form $\score(\estim,z) = h(z) - h(\estim) + h'(\estim) (\estim-z)$, where $h$ is strictly convex with subgradient $h'$.
We remark that using $h(z) = z^2$ leads to the squared error.
The value-at-risk ($\VaR_{\alpha}$) at level $\alpha \in (0,1)$, and thus any $\alpha$-quantile, is also 1-elicitable and a strictly consistent scoring function is necessarily written as $\score(\estim,z) = ( \Ind_{\{\estim \leq z\}} - \alpha) (h(\estim) - h(z))$ for a nondecreasing $h$.
The conditional value-at-risk ($\CVaR_{\alpha}$) at level $\alpha \in (0,1)$ is not 1-elicitable, but rather 2-elicitable along side the value-at-risk.
One characterization of a strictly consistent scoring function for the pair $(\VaR_{\alpha}, \CVaR_{\alpha})$ is
\begin{equation}
    \label{eq:score-CVaR}
    \score(\estim_1, \estim_2, z) = \log \left( \frac{\estim_2+C}{z+C} \right) - \frac{\estim_2}{\estim_2+C} + \frac{ \estim_1 \left( \Ind_{\{z \leq \estim_1\}} - \alpha \right) + z \, \Ind_{\{z > \estim_1\}} }{(\estim_2+C) (1-\alpha)},
\end{equation}
where $C > 0$ and $\estim_1 \leq \estim_2$, i.e. the $\CVaR_{\alpha}$ must be greater than $\VaR_{\alpha}$.
In addition, as noted by \cite{frongillo2015vector}, we may construct a strictly consistent scoring function for a vector of elicitable components using the scoring functions of each component.

\begin{remark}
As we can observe, there exist infinitely many characterizations of strictly consistent scoring functions for $k$-elicitable mappings. In this paper, we do not investigate how different characterizations may affect optimization performances and potentially improve the convergence speed of our RL algorithm.
\end{remark}

\section{Problem Setup}
\label{sec:problem-setup}

In this section, we introduce and rigorously define the class of \RL{} problems we aim to solve. We describe each problem as an \emph{agent} who tries to learn an optimal behavior, or \emph{agent's policy}, that attains the minimum in a given objective function by interacting with a certain \emph{environment} in a model-agnostic manner.

Let $\statespace$ and $\actionspace$ be arbitrary state and action spaces respectively, and let $\costspace \subset \Reals$ be a cost space.
We represent the environment as a MDP with the tuple $(\statespace, \actionspace, \costfunc, \PP)$, where $\costfunc(\state, \action, \statedum) \in \costspace$ is a cost function and $\PP(\state_{\timeidx+1} = \statedum \;\ifnum\currentgrouptype=16 \middle\fi|\; \state_{\timeidx}, \action_{\timeidx})$ characterize the transition probabilities.
The transition probability is assumed stationary, although time may be a component of the state. In what follows, we take the augmented state space where the first dimension represents time.
An episode consists of a sequence of $\timelength+1$ periods, where $\timelength \in \Nats$ is known and finite.
We often assume that the periods refer to fixed intervals, but the framework may be extended to periods of arbitrary length.
At each period, the agent begins in a state $\state_{\timeidx} \in \statespace$, takes an action $\action_{\timeidx} \in \actionspace$ according to a deterministic policy $\policy: \statespace \rightarrow \actionspace$, moves to the next state $\state_{\timeidx+1} \in \statespace$, and receives a cost $\costfunc_{\timeidx} = \costfunc(\state_{\timeidx}, \action_{\timeidx}, \state_{\timeidx+1}) < \infty$.
We view the cost function as a deterministic mapping of the states and actions, but we can easily generalize to include other sources of randomness.\footnote{Considering randomized policies would require additional care to establish a proper dynamic programming principle and extending to this case is outside the scope of this paper. See, e.g., \cite{cheng2025risk}.}

We consider strongly time-consistent dynamic risk measures $\{\robust_{\timeidx,\timelength}\}_{\timeidx \in \timespace}$ with one-step conditional risk measures that are robust distortion risk measures with piecewise constant distortion functions, as defined in \cref{ssec:dynamic-risk}.
We aim to solve $(\timelength+1)$-period robust risk-aware \RL{} problems of the form
\begin{equation}
	\min_{\policy} \, \robust_{0,\timelength}^{\vec{\epsilon}, \vec{\gamma}} \Big( \{ \costfunc^{\policy}_{\timeidx} \}_{\timeidx \in \timespace} \Big) = \min_{\policy} \, 
	\robust_{0}^{\epsilon_{\state_{0}},\gamma_{\state_{0}}} \bigg(\costfunc^{\policy}_{0} +
	\robust_{1}^{\epsilon_{\state_{1}},\gamma_{\state_{1}}} \Big(\costfunc^{\policy}_{1} +
	\cdots +	\robust_{\timelength}^{\epsilon_{\state_{\timelength}},\gamma_{\state_{\timelength}}} \big(\costfunc^{\policy}_{\timelength}
	\big) \cdots \Big) \bigg), \tag{P}\label{eq:problem}
\end{equation}
where $\costfunc^{\policy}_{\timeidx} = \costfunc(\state_{\timeidx}, \policy(\state_{\timeidx}), \state_{\timeidx+1})$ is a bounded $\Ff_{\timeidx+1}$-measurable random cost modulated by the policy $\policy$.
State-dependent distortion functions $\gamma_{\state_{\timeidx}}$ and tolerances $\epsilon_{\state_{\timeidx}}$ may be used to illustrate an agent that, for instance, slowly reduces her model uncertainty as she gets closer to the terminal period or drastically changes her risk preferences in less favorable states. 
In this paper, we prove the results under the assumption of state-dependent distortions and tolerances, but focus on simpler state-independent dynamic risk measures in the experimental section.

As opposed to the typical literature on robust MDPs, we consider uncertainty sets directly on the distribution of costs-to-go instead of the transition probabilities. This allows us to capture dependence on the factors that generates them, such as the transition probabilities, the agent's actions, and other moment constraints through the form of $\varphi$. Moreover, in a \RL{} setting, the agent can sample transitions but does not necessarily know the true underlying dynamics. In this sense, our formulation makes no explicit assumptions on how the uncertainty set interacts with $\PP(\statedum \;\ifnum\currentgrouptype=16 \middle\fi|\; \state_{\timeidx}, \action_{\timeidx})$.

\begin{remark}
Here, we do not subtract the dynamic robust risk of zero to obtain a weak recursive dynamic risk measure \citep[see Theorem 4 of][]{moresco2025uncertainty}, as the risk of zero does not depend on the policy parameters and, hence, does not play a role in the optimization procedure.
\end{remark}

Given \cref{eq:problem}, we define the \emph{value function} for an agent as the running risk-to-go
\begin{equation*}
    \valuefunc_{\timeidx}(\state; \policy) :=
    \robust_{\timeidx}^{\epsilon_{\state_{\timeidx}},\gamma_{\state_{\timeidx}}} \bigg(\costfunc_{\timeidx}^{\policy} +
    \robust_{\timeidx+1}^{\epsilon_{\state_{\timeidx+1}},\gamma_{\state_{\timeidx+1}}} \Big(\costfunc_{\timeidx+1}^{\policy} +
    \dots +
    \robust_{\timelength}^{\epsilon_{\state_{\timelength}},\gamma_{\state_{\timelength}}} \big(\costfunc_{\timelength}^{\policy} \big) 
    \Big) \biggm| \state_{\timeidx}=\state \bigg),
\end{equation*}
for all $\state \in \statespace$ and $\timeidx \in \timespace$.
It gives the (time-consistent) dynamic risk of the sequence of future costs for the agent following the policy $\policy$ at a certain time $\timeidx$ when being in a specific state $\state$.
Using \cref{def:robust-risk}, the dynamic programming equations for a specific policy $\policy$ are
\begin{equation*}
	\valuefunc_{\timeidx}(\state;\policy)
    =
    \robust_{\timeidx}^{\epsilon_{\state},\gamma_{\state}} \Big(\costfunc_{\timeidx}^{\policy} +
    \valuefunc_{\timeidx+1}(\state_{\timeidx+1}^{\policy};\policy) \Bigm| \state_{\timeidx}=\state \Big)
    =
    \esssup_{\rv_{\timeidx}^{\adv} \in \varphi_{\rv_{\timeidx}^{\policy}}^{\epsilon_{\state}}} \,
    \Big\langle \gamma_{\state} , \breve{F}_{\adv,\timeidx}(\cdot|s) \Big\rangle,
\end{equation*}
where $\rv_{\timeidx}^{\policy} = \costfunc_{\timeidx}^{\policy} + \valuefunc_{\timeidx+1} (\state_{\timeidx+1}^{\policy};\policy)$, $\breve{F}_{\adv,\timeidx}(\cdot|s)$ is the conditional quantile of $\rv_{\timeidx}^{\adv}$ given $\state_{\timeidx}=\state$, and $\valuefunc_{\timelength+1} = 0$.
We apply the dynamic programming principle (DPP) to recover a Bellman-like equation for the value function:
\begin{equation}
    \label{eq:bellman-value}
    \valuefunc_{\timeidx}(\state;\policy^{*}) =
    \min_{\action \in \actionspace} \;
    \esssup_{\rv_{\timeidx}^{\adv} \in \varphi_{\rv_{\timeidx}^{\action,\policy^{*}}}^{\epsilon_{\state}}}\,
    \Big\langle \gamma_{\state} , \breve{F}_{\adv,\timeidx}(\cdot|s) \Big\rangle,
\end{equation}
where $\rv_{\timeidx}^{\action,\policy^{*}} = \costfunc(\state_{\timeidx}, \action, \state_{\timeidx+1}) + \valuefunc_{\timeidx+1}(\state_{\timeidx+1};\policy^{*})$.
The previous equation indicates the optimal policy for the agent at any point in state with a recursive equation involving the future costs, composed of the cost at current time $\timeidx$, the running risk-to-go at the next time $\timeidx+1$, and an adversary who distorts both to get the worst performance.

Alternatively, we define the \emph{quality function} or \emph{Q-function}, which effectively represents the running risk-to-go for an agent starting in any state-action tuple and thereafter following the policy $\policy$, as
\begin{equation*}
    \qfunc_{\timeidx}(\state, \action; \policy) :=
    \robust_{\timeidx}^{\epsilon_{\state_{\timeidx}},\gamma_{\state_{\timeidx}}} \bigg(\costfunc(\state_{\timeidx}, \action_{\timeidx}, \state_{\timeidx+1}) +
    \robust_{\timeidx+1}^{\epsilon_{\state_{\timeidx+1}},\gamma_{\state_{\timeidx+1}}} \Big(\costfunc_{\timeidx+1}^{\policy} +
    \dots +
    \robust_{\timelength}^{\epsilon_{\state_{\timelength}},\gamma_{\state_{\timelength}}} \big(\costfunc_{\timelength}^{\policy} \big) 
    \Big) \biggm| \state_{\timeidx}=\state, \, \action_{\timeidx}=\action \bigg),
\end{equation*}
for all $\state \in \statespace$, $\timeidx \in \timespace$ and $\action_{\timeidx} \in \actionspace$.
Using the DPP for the value function in \cref{eq:bellman-value}, we have a similar Bellman-like equation for the Q-function.

The goal is to minimize the value function $\valuefunc_{\timeidx}(\state; \policy) = \qfunc_{\timeidx}(\state, \policy(\state); \policy)$ over policies $\policy$, which also requires maximizing the worst-case risk over $\adv$ in the uncertainty set and estimating the value function itself.
We propose to use parametric approximators for the different components we optimize, and thus we aim to estimate the Q-function
\begin{equation}
    \qfunc_{\timeidx}(\state, \action;\policyparams) =
    \esssup_{\rv_{\timeidx}^{\adv} \in \varphi_{\rv_{\timeidx}^{\policyparams}}^{\epsilon_{\state}}} \,
    \Big\langle \gamma_{\state} , \breve{F}_{\adv,\timeidx}(\cdot|s,a) \Big\rangle,
    \label{eq:q-function-algo}
\end{equation}
where $\policyparams$ are the policy parameters.
For compact notation, we denote costs-to-go by $\rv^{\policyparams}_{\timeidx} := \costfunc_{\timeidx}+\qfunc_{\timeidx+1}(\state_{\timeidx+1}, \policy^{\policyparams}(\state_{\timeidx+1});\policyparams)$ and their conditional \CDF{} by $F_{\policyparams,\timeidx}(z|\state,\action):= F_{\rv^{\policyparams}_{\timeidx}|_{\state_{\timeidx}=\state,\action_{\timeidx}=\action}}(z)$.
In the next sections, we determine the distribution of the worst-case cost-to-go $\rv_{\timeidx}^{\adv}$ conditionally on a state-action pair $(\state,\action)$ for different uncertainty sets.
This allows us to design random variables that maximize the one-step conditional risk measure within the Q-function while remaining distributionally close to the original cost-to-go according to the appropriate uncertainty set.

\subsection{Wasserstein uncertainty set}
\label{ssec:U-set-wass}

In this section, we investigate dynamic robust risk measures, where the one-step conditional risk measures are distortion risk measures with uncertainty sets of the form in \cref{eq:U-set-wass}. More precisely, we restrict the random variable's distribution to lie within a Wasserstein ball within the original distribution.
The next result is similar to Theorem 3.9 of \cite{pesenti2023portfolio}, here, however, we reformulate to account for conditional distributions and study the problem without any terminal or copula constraints.
The proof of \cref{thm:optimal-quantile-fn} is deferred to \cref{sec:proof-optimal-quantile-fn}.
\begin{theorem}
\label{thm:optimal-quantile-fn}
    Let $F^{\uparrow}$ be the isotonic projection of a function $F$, more precisely its projection onto the set of quantile functions $F^{\uparrow} := \argmin_{G\in\breve{\FF}} \, \big\Vert G-F \big\Vert^{2}$, where $\breve{\FF} := \{F\in\LL^{2}([0,1]) \, : \, F \text{ is nondecreasing and left-continuous} \}$. Further, consider the Q-function in \cref{eq:q-function-algo}, where its uncertainty set is of the form in \cref{eq:U-set-wass}.
    The quantile function of the optimal random variable in the optimization problem $\qfunc_{\timeidx}(\state,\action;\policyparams)$ is given by
    \begin{equation*}
        \breve{F}^{*}_{\adv,\timeidx}(\cdot|\state,\action) = \bigg( \breve{F}_{\policyparams,\timeidx}(\cdot|\state,\action) + \frac{\gamma_{\state}(\cdot)}{2\lambda^{*}} \bigg)^{\uparrow},
    \end{equation*}
    where $\lambda^{*}>0$ is such that $\big\Vert \breve{F}^{*}_{\adv,\timeidx}(\cdot|\state,\action) - \breve{F}_{\policyparams,\timeidx}(\cdot|\state,\action) \big\Vert =\epsilon_{\state}$.
\end{theorem}

Equipped with the previous result, finding the optimal quantile function may be computationally intensive, because it requires repeatedly solving the optimization problem in \cref{thm:optimal-quantile-fn} to find the optimal $\lambda^{*}$ for any pair $(\state,\action)\in\statespace\times\actionspace$.
One solution to alleviate this issue consists of performing parallel computations to accelerate the process.
Another approach is to work with a smaller class of dynamic risk measures.
Indeed, \cref{thm:optimal-quantile-fn} may be simplified if we consider one-step conditional distortion risk measures that are coherent. We specify this observation in the result below.

\begin{corollary}
\label{thm:optimal-quantile-fn-analytical}
   Consider the Q-function in \cref{eq:q-function-algo}, where its uncertainty set is of the form in \cref{eq:U-set-wass} and the distortion function $\gamma_{\state}$ of each one-step conditional risk measure is nondecreasing.
   The quantile function of the optimal random variable in the optimization problem $\qfunc_{\timeidx}(\state,\action;\policyparams)$ is
    \begin{equation*}
        \breve{F}^{*}_{\adv,\timeidx}(\cdot|\state,\action) = \breve{F}_{\policyparams,\timeidx}(\cdot|\state,\action) + \frac{\epsilon_{\state} \gamma_{\state}(\cdot)}{ \left\Vert \gamma_{\state} \right\Vert}.
    \end{equation*}
\end{corollary}
\begin{proof}
    The result follows from \cref{thm:optimal-quantile-fn}.
    Since both $\breve{F}_{\policyparams,\timeidx}$ and $\gamma_{\state}$ are nondecreasing, the isotonic projection equals itself and we recover $\lambda^{*} = \Vert \gamma_{\state} \Vert / 2 \epsilon_{\state}$ from the constraint.
\end{proof}

Using nondecreasing distortion functions, \cref{thm:optimal-quantile-fn-analytical} implies that
\begin{equation}
    \qfunc_{\timeidx}(\state, \action;\policyparams)
    = \Big\langle \gamma_{\state} , \breve{F}_{\policyparams,\timeidx}(\cdot|\state,\action) \Big\rangle + \epsilon_{\state} \left\Vert \gamma_{\state} \right\Vert = \Big\langle \gamma_{\state} , \breve{F}_{ \rv_{\timeidx}^{\policyparams} + \epsilon_{\state} \Vert \gamma_{\state} \Vert, \timeidx}(\cdot|\state,\action) \Big\rangle . \label{eq:q-U-set-wass}
\end{equation}
As conditional distortion risk measures are cash-additive, the $\Ff_{t}$-measurable shift $\epsilon_{\state} \Vert \gamma_{\state} \Vert$ in \cref{eq:q-U-set-wass} may be included as part of the cost function $\costfunc_{\timeidx}$, regardless of the tolerance being constant or state-dependent.
Therefore, for dynamic distortion risk measures with nondecreasing distortion functions and uncertainty sets of the form in \cref{eq:U-set-wass}, robustness is equivalent to modulating the cost function.
In fact, this observation may be extended to other dynamic monetary risk measures with uncertainty sets induced only by semi-norms on the space of random variables.
Furthermore, with state-independent parameters, the robust and non-robust optimal policies remain identical, because the shift $\epsilon \Vert\gamma\Vert_2$ does not depend whatsoever on the policy parameters $\policyparams$.

\begin{remark}
For this class of problems, the structure of the resulting actor-critic algorithm resembles other deep deterministic policy gradient algorithms found in the literature \citep[see e.g.][]{lillicrap2015continuous,marzban2023deep}. We leave for future works an investigation of the algorithm performances, and instead inspect the effects of modifying the form of the conditional uncertainty set.
\end{remark}

\subsection{Wasserstein uncertainty set with moment constraints}
\label{ssec:U-set-wass-mom}

As explained in the previous section, an uncertainty set \cref{eq:U-set-wass} with a constant tolerance and distortion function leads to identical optimal policies in both the robust and non-robust cases. To overcome this situation, we can either include a state-dependent tolerance or modify the uncertainty set. In this section, we explore the latter option by considering dynamic robust distortion risk measures with uncertainty sets of the form in \cref{eq:U-set-wass-mom}.
More precisely, we restrict the random variable's distribution to lie within a Wasserstein ball from the original distribution and have the same first and second moments.\footnote[1]{Constraining only the first moment leads to strategies that are also identical to the non-robust case as in \cref{eq:U-set-wass}. Constraining only the second moment, however, does lead to distinct policies and the algorithm we derive can easily be modified to this case.}

Next, we cast in a dynamic setting Theorem 3.1 of \cite{bernard2024robust}, where they derive explicit bounds for static distortion risk measures when the random variable's distribution has known first two moments and lies within a 2-Wasserstein ball from a reference distribution.
The proof, using a Lagrange multiplier technique, is provided in \cref{sec:proof-optimal-quantile-fn-mom}.
\begin{theorem}
   \label{thm:optimal-quantile-fn-mom}
   Consider the Q-function in \cref{eq:q-function-algo}, where its uncertainty set is of the form in \cref{eq:U-set-wass-mom} and the distortion function $\gamma_{\state}$ of each one-step conditional risk measure is nondecreasing.
   The quantile function of the optimal random variable in the optimization problem $\qfunc_{\timeidx}(\state,\action;\policyparams)$ is then given by
    \begin{equation*}
        \breve{F}^{*}_{\adv,\timeidx}(u|\state,\action) = \mu + \frac{\lambda^{*} \big(\breve{F}_{\policyparams,\timeidx}(u|\state,\action) - \mu\big) + \gamma_{\state}(u) - 1}{b_{\lambda^{*}}},
    \end{equation*}
    where $K = \sigma^2 - \frac{\epsilon_{\state}^2}{2}$, $\mu = \big\langle \breve{F}_{\policyparams,\timeidx}(\cdot|\state,\action), 1 \big\rangle$, $\sigma^2 = \big\Vert\breve{F}_{\policyparams,\timeidx}(\cdot|\state,\action)\big\Vert^2 - \mu^2$, $\sigma_{\gamma}^2 = \big\Vert \gamma_{\state} \big\Vert^2 - 1$,
    \begin{gather*}
       b_{\lambda^{*}} = \frac{\sqrt{ (\lambda^{*}\sigma)^2 + \sigma_{\gamma}^2 + 2\lambda^{*} \left(\big\langle \breve{F}_{\policyparams,\timeidx}(\cdot|\state,\action), \gamma_{\state} \big\rangle - \mu\right) }}{\sigma}, \\
       \lambda^{*} = \frac{- 2 \big(\big\langle \breve{F}_{\policyparams,\timeidx}(\cdot|\state,\action), \gamma_{\state} \big\rangle - \mu\big) + \sqrt{\Delta}}{2\sigma^2}, \quad
       \Delta = \frac{4 K^2}{K^2 - \sigma^4} \bigg( \Big( \big\langle \breve{F}_{\policyparams,\timeidx}(\cdot|\state,\action), \gamma_{\state} \big\rangle - \mu \Big)^2 - \sigma^2 \sigma_{\gamma}^2 \bigg).
    \end{gather*}
    The optimal solution remains valid with $\lambda^{*} = 0$ if the uncertainty tolerance $\epsilon_{\state}$ is such that
    \begin{equation}
        \epsilon_{\state}^2 > 2 \sigma^2 \left( 1 - \frac{ \big\langle \breve{F}_{\policyparams,\timeidx}(\cdot|\state,\action), \gamma_{\state} \big\rangle - \mu }{\sigma \sigma_{\gamma}} \right).\label{eq:larger-eps}
    \end{equation}
\end{theorem}

We note that both $\lambda^{*}$ and $b_{\lambda^{*}}$ of the optimal quantile function in \cref{thm:optimal-quantile-fn-mom} depend non-trivially on the quantile function $\breve{F}_{\rv_{\timeidx}^{\policyparams}}$.
This uncertainty set differs from the previous case, because even with a constant tolerance $\epsilon$ and distortion function $\gamma$, the robust optimal policy may be different than the non-robust optimal policy due to the intricate dependence on the policy parameters $\policyparams$.

\subsection{Deterministic Gradient}
Next, we derive the analytical expression of the gradient of the value function $\valuefunc_{\timeidx}(\state;\policyparams) = \qfunc_{\timeidx}(\state, \policy^{\policyparams}(\state);\policyparams)$. The proof is provided in \cref{sec:proof-gradient-V}.
\begin{theorem}
\label{thm:gradient-V}
    Consider the setup of \cref{thm:optimal-quantile-fn-mom}. The gradient of the value function with an uncertainty set of the form in \cref{eq:U-set-wass-mom} is
    \begin{equation*}
    \begin{split}
        \grad{\policyparams} \valuefunc_{\timeidx}(\state;\policyparams)
        &= \grad{\policyparams} \policy^{\policyparams}(\state) \Bigg( \grad{a} \qfunc_{\timeidx}(\state, a; \policyparams)\Big|_{a = \policy^{\policyparams}(\state)} \\
        &\quad - \frac{b_{\lambda^{*}} - \lambda^{*}}{b_{\lambda^{*}}} \EE_{\timeidx,\state} \left[ \Big((b_{\lambda^{*}} - \lambda^{*}) (\rv_{\timeidx}^{\policyparams} - \mu) + 1\Big) \; \frac{ \grad{a} F_{\policyparams,\timeidx}(x|\state,a)}{\grad{x} F_{\policyparams,\timeidx}(x|\state,a)} \bigg|_{(x,a)=(\rv_{\timeidx}^{\policyparams},\policy^{\policyparams}(\state))} \right] \Bigg).
    \end{split}
    \end{equation*}
\end{theorem}

When the uncertainty tolerance $\epsilon_{\state}$ tends to zero, the gradient of the value function, and thus the Q-function, reduces to the well-known deterministic policy gradient update rule from \cite{silver2014deterministic}. We prove that the usual deterministic gradient formula is a limiting case, as $\epsilon_{\state}$ approaches zero, of \cref{thm:gradient-V}.
\begin{corollary}
    The gradient of the value function with an uncertainty set of the form in \cref{eq:U-set-wass-mom} satisfies
    \begin{equation*}
        \lim_{\epsilon_{\state} \downarrow 0} \grad{\policyparams} \valuefunc_{\timeidx}(\state;\policyparams)
        = \lim_{\epsilon_{\state} \downarrow 0} \grad{\policyparams} \qfunc_{\timeidx}(\state, \policy^{\policyparams}(\state);\policyparams)
        = \grad{a} \qfunc_{\timeidx}(\state, a; \policyparams)\Big|_{a = \policy^{\policyparams}(\state)} \grad{\policyparams} \policy^{\policyparams}(\state).
    \end{equation*}
\end{corollary}
\begin{proof}
     From \cref{thm:optimal-quantile-fn-mom}, we observe that as the uncertainty tolerance decreases, the optimal quantile function $\breve{F}^{*}_{\adv,\timeidx}$ converges to the quantile function of the costs-to-go $\breve{F}_{\policyparams,\timeidx}$. More precisely, (i) $K \rightarrow \sigma^2$, (ii) $\Delta, \lambda^{*}, b_{\lambda^{*}} \rightarrow \infty$, and (iii) $\lambda^{*}/b_{\lambda^{*}} \rightarrow 1$ as $\epsilon_{\state} \downarrow 0$. This leads to $(b_{\lambda^{*}} - \lambda^{*})/b_{\lambda^{*}} \rightarrow 0$ as $\epsilon_{\state} \downarrow 0$, which concludes the proof.
\end{proof}

\section{Algorithm}
\label{sec:algorithm}
In this section, we highlight the main steps of our learning algorithm and provide details on its implementation.
The full algorithm is provided in \cref{algo:robust-dynamic-RL} with detailed steps for the critic (\cref{ssec:adversary,ssec:critic}) and actor (\cref{ssec:actor}).
As well, our Python code is publicly available in the \href{https://github.com/acoache/RL-DynamicRobustRisk}{Github repository RL-DynamicRobustRisk}.

\usetikzlibrary{shapes,arrows,trees,automata,positioning,matrix,decorations.pathreplacing,calc}

\tikzstyle{input}=[shape=circle,draw=mblue,fill=mblue!20,minimum width=1.0cm, line width=0.8pt]
\tikzstyle{output}=[shape=circle,draw=mgreen,fill=mgreen!20,minimum width=1.0cm, line width=0.8pt, inner sep=0.01cm]
\tikzstyle{ffnn}=[shape=rectangle,draw=mred,fill=mred!10,minimum width=1.75cm, minimum height=0.85cm]

% \begin{figure}[htbp]
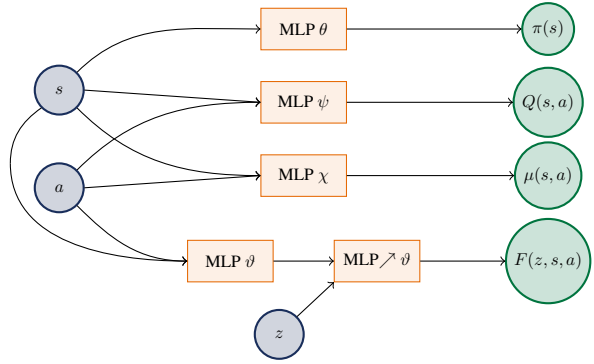
\begin{wrapfigure}{r}{0.5\textwidth}
\centering
    \begin{tikzpicture}[scale=0.65,every node/.style={transform shape}]
    % inputs
    \node[input] (in0) at (1,4) {$s$};
    \node[input] (in1) at (1,2) {$a$};
    \node[input] (in2) at (4.5,-0.75) {$z$};
    \node[] (temp) at (0,2.5) {};
    % FFNNs
    \node[ffnn, align=center] (theta) at (5,5.25) {\MLP{} $\policyparams$};
    \node[ffnn, align=center] (varphi) at (5,3.75) {\MLP{} $\qparams$};
    \node[ffnn, align=center] (mu) at (5,2.25) {\MLP{} $\muparams$};
    \node[ffnn, align=center] (vartheta) at (3.5,0.75) {\MLP{} $\fparams$};
    \node[ffnn, align=center] (vartheta2) at (6.5,0.75) {\MLP{}$\nearrow$ $\fparams$};
    % outputs
    \node[output] (out0) at (9,5.25) {$\policy(\state)$};
    \node[output] (out1) at (9,3.75) {$\qfunc(\state,\action)$};
    \node[output] (out2) at (9,2.25) {$\mu(\state,\action)$};
    \node[output] (out3) at (9,0.75) {$\ffunc(z,\state,\action)$};

    \draw [->] (in0.north east) to [in=180] (theta.west);
    \draw [->] (in0.east) to (varphi.west);
    \draw [->] (in0.south east) to [out=-45,in=180] (mu.west);
    \draw [-] (in0.south west) to [out=215, in=90] (temp.center);
    \draw [->] (temp.center) to [out=270, in=180] (vartheta.west);
    \draw [->] (in1.north east) to [out=45,in=180] (varphi.west);
    \draw [->] (in1.east) to (mu.west);
    \draw [->] (in1.south east) to [out=-45,in=180] (vartheta.west);
    \draw [->] (in2.north east) to (vartheta2.south west);
    \draw [->] (theta.east) to (out0.west);
    \draw [->] (varphi.east) to (out1.west);
    \draw [->] (mu.east) to (out2.west);
    \draw [->] (vartheta.east) to (vartheta2.west);
    \draw [->] (vartheta2.east) to (out3.west);
    \end{tikzpicture}
% \caption{Neural network structures for the various components, i.e. the policy $\policy^{\policyparams}$, the Q-function $\qfunc_{\timeidx}(\state,\action; \policyparams)$, the conditional first moment $\EE[ \rv^{\policyparams}_{\timeidx}| \state_{\timeidx},\action_{\timeidx}]$, and the conditional CDF of costs-to-go given state-action pairs $F_{\rv_{\timeidx}^{\policyparams}|_{\state_{\timeidx},\action_{\timeidx}}}$. Here, \MLP{}$\nearrow$ denotes the constrained \MLP{} to ensure monotonicity.}\anthony{changed wrapfigure and shorter caption}
\caption{Neural network structures for the various components. Here, \MLP{}$\nearrow$ denotes the constrained \MLP{} to ensure monotonicity.}
\label{fig:tikz-policyANN}
\end{wrapfigure}
% \end{figure}
Recall that we aim to optimize $\valuefunc_{\timeidx}(\state; \policy) = \qfunc_{\timeidx}(\state, \policy(\state); \policy)$ over policies $\policy$.
To do so, we propose an actor-critic style \citep{konda2000actor} algorithm, known for their ability to find optimal policies using low variance gradient estimates and their good convergence properties.
Our algorithm aims to learn four functions in an alternating manner. The \emph{critic} estimates the conditional \CDF{} of the costs-to-go $\rv^{\policyparams}_{\timeidx}$ given the current state-action pair, the conditional first moment of $\rv^{\policyparams}_{\timeidx}$, and the Q-function with a deep composite model using the elicitability of the dynamic risk, while the \emph{actor} updates the current policy via a policy gradient method.
In addition, our proposed approach, similar to the deep deterministic policy gradient algorithm \citep{lillicrap2015continuous}, is off-policy, in the sense that the agent learns the Q-function and the conditional distribution of $\rv^{\policyparams}_{\timeidx}$ independently of the agent's actions.
Off-policy \RL{} algorithms can lead to better data efficiency, by reusing observations with a replay buffer, and thus faster convergence speed.

\begin{algorithm}[htbp]
\caption{Actor-Critic for Dynamic Robust Distortion Risk Measures}
\label{algo:robust-dynamic-RL}
{\footnotesize
\KwIn{Main networks for $\ffunc^{\fparams},\qfunc^{\qparams},\mufunc^{\muparams},\policy^{\policyparams}$;
    number of epochs $K^{\fparams},K^{\qparams},K^{\muparams},K^{\policyparams}$;\newline
    mini-batch sizes $B^{\fparams},B^{\qparams},B^{\muparams},B^{\policyparams}$;
    initial learning rates $\eta^{\fparams},\eta^{\qparams},\eta^{\muparams},\eta^{\policyparams}$;\newline
    parameters for dynamic risk $\vec{\gamma},\vec{\epsilon}$, exploration $p_{\text{ex}},\Nn$, and soft target $\tau$}
Instantiate and initialize the environment, optimizers and schedulers\;
Initialize target networks for $\ffunc,\qfunc,\mufunc,\policy$ with $(\tilde{\fparams}, \tilde{\qparams}, \tilde{\muparams},\tilde{\policyparams}) \leftarrow (\fparams, \qparams, \muparams, \policyparams)$\;
\For{each iteration $i = 1, 2, \ldots$}{
    Simulate full trajectories induced by the policy and exploration noise $\policy^{\policyparams} + \Nn$\;
    Generate a partition $\{y_i\}_{i}$ covering the span of costs-to-go\;
    \Repeat{convergence of both steps is achieved}{
        \For(\tcp*[f]{Critic}){each epoch $k = 1, \ldots, K^{\fparams}$}{
            Zero out the gradients of $F^{\fparams}$\;
            Sample a mini-batch of $B^{\fparams}$ full trajectories and compute the loss $\Ll^{\fparams}$ in \cref{eq:loss-adversary}\;
            Update $\fparams$ by performing an Adam optimization step and tune the learning rate $\lrate^{\fparams}$\;
            \lIf{convergence is achieved}{\Break}   
        }
        \For{each epoch $k = 1, \ldots, K^{\muparams}$}{
            Zero out the gradients of $\mufunc^{\muparams}$\;
            Sample a mini-batch of $B^{\muparams}$ full trajectories and compute the loss $\Ll^{\muparams}$ in \cref{eq:loss-mu}\;
            Update $\muparams$ by performing an Adam optimization step and tune the learning rate $\lrate^{\muparams}$\;
            \lIf{convergence is achieved}{\Break}
        }
        \For{each epoch $k = 1, \ldots, K^{\qparams}$}{
            Zero out the gradients of $\qfunc^{\qparams}$\;
            Sample a mini-batch of $B^{\qparams}$ full trajectories and compute the loss $\Ll^{\qparams}$ in \cref{eq:loss-critic}\;
            Update $\qparams$ by performing an Adam optimization step and tune the learning rate $\lrate^{\qparams}$\;
            \lIf{convergence is achieved}{\Break}
        }
    }
    Simulate full trajectories induced by the policy $\policy^{\policyparams}$\;
    \For(\tcp*[f]{Actor}){each epoch $k = 1, \ldots, K^{\policyparams}$}{
        Zero out the gradients of $\policy^{\policyparams}$\;
        Sample a mini-batch of $B^{\policyparams}$ full trajectories and compute the loss $\Ll^{\policyparams}$ in \cref{eq:loss-actor}\;
        Update $\policyparams$ by performing an Adam optimization step, and tune the learning rate $\lrate^{\policyparams}$\;
    }
    Decay the exploration probability $p_{\text{ex}}$\;
    Update target networks using $(\tilde{\fparams}, \tilde{\qparams}, \tilde{\muparams},\tilde{\policyparams}) \leftarrow \tau \cdot (\fparams, \qparams, \muparams, \policyparams) + (1-\tau) \cdot (\tilde{\fparams}, \tilde{\qparams}, \tilde{\muparams},\tilde{\policyparams})$\;
}
\KwOut{Approximation of the optimal policy $\policy^{\policyparams}$ and corresponding Q-function $\qfunc^{\qparams}$}
}
\end{algorithm}

We use artificial neural networks (\ANN{}s), known to be universal approximators, as function approximations for both the critic and actor components of our algorithms. \ANN{}s excel at modeling complicated functions through layered compositions of simple functions and avoiding the curse of dimensionality issue when representing nonlinear functions in high dimensions.
We consider the following fully-connected multi-layered feed forward \ANN{} structures, or multilayer perceptron (\MLP{}), as illustrated in \cref{fig:tikz-policyANN}.
We characterize the policy by an \ANN{}, denoted by $\policy^{\policyparams}:\statespace \rightarrow \actionspace$, which takes a state $\state_{\timeidx}$ as input and outputs a deterministic action.
The Q-function and conditional expectation of costs-to-go are characterized by $\qfunc^{\qparams},\mufunc^{\muparams}:\statespace \times \actionspace \rightarrow \Reals$.
We also characterize the \CDF{} of the costs-to-go by $\ffunc^{\fparams}: \statespace \times \actionspace \times \Reals \rightarrow [0,1]$, where $\ffunc^{\fparams}(\state,\action,z)$ gives the conditional \CDF{} evaluated at $z$ of the costs-to-go given the state-action pair $(\state,\action)$. As well, for layers that are descendant of $z$, we constrain the weights of the \ANN{} to nonnegative values and use monotonic activation function to ensure a nondecreasing mapping \citep[see e.g.][]{sill1997monotonic}.
The exact structure of those neural networks in terms of number of nodes and layers is ultimately application-dependent and, consequently, described in \cref{sec:experiments}.

\begin{remark}
    One may calculate the conditional first moment of the costs-to-go using the conditional \CDF{} and ignore the \ANN{} $\mufunc^{\muparams}$. In practice, however, we observe that having separate networks for both the \CDF{} and first moment produces more stable results.
\end{remark}
\begin{remark}
    In our current algorithm, we favor simpler neural network structures as a proof of concept. We do not address here how different \ANN{} structures, such as recurrent neural networks, convolutional neural networks, concatenating inputs at different layers or using pre-trained foundation models, may better capture the hidden underlying patterns for more complex applications and, as an end result, improve the learning speed.
\end{remark}

The following sections describe in more detail the derivation of the updates rules, as well as some implementation technicalities.
We assume that we have access to mini-batches of $B$ transitions induced by the exploratory policy $\policy^{\policyparams} + \Nn$, i.e. the current policy with some white noise $\Nn$, whether by generating transitions from the simulation engine or by sampling transitions from a replay buffer, which we denote by $\big(\state_{\timeidx,b}; \action_{\timeidx,b}; \costfunc_{\timeidx,b}; \state_{\timeidx+1,b} \big)$ for $b=1,\ldots,B$.

\subsection{CDF of Costs-to-go}
\label{ssec:adversary}

The objective consists of estimating with the \ANN{} $\ffunc^{\fparams}$ the \CDF{} of the costs-to-go $\rv^{\policyparams}_{\timeidx}$ conditionally on the current state-action pair $(\state,\action)$, formally $F_{\policyparams,\timeidx}(z|\state,\action)$.
If one knows the conditional distribution of the costs-to-go, then in view of \cref{thm:optimal-quantile-fn,thm:optimal-quantile-fn-mom}, one understands how to perturb the costs-to-go in order to maximize the distortion risk. In this procedure, we suppose that all other networks, i.e., Q-function $\qfunc^{\qparams}$, first moment $\mufunc^{\muparams}$ and policy $\policy^{\policyparams}$, are fixed.

We propose to use scoring rules in order to obtain an estimation of this conditional \CDF{}.
A strictly proper scoring rule for probabilistic forecasts is the equivalent of a strictly consistent scoring function for point forecasts.
We refer the reader to \cite{gneiting2007strictly} for a thorough discussion and numerous examples of proper scoring rules on general sample spaces.
The continuous ranked probability score is known to be a strictly proper scoring rule for the class of \CDF{}s with finite first moments.
For any random variable $\rv$, its \CDF{} $F_{\rv}$ is
\begin{equation*}
    F_{\rv} = \argmin_{F \in \FF} \EE_{\rv \sim F_{\rv}} \Big[\score(F, \rv) \Big],
\end{equation*}
with the strictly proper scoring rule $\score(F, z) = \textstyle\int_{\Reals} \big(F(y) - \Ind_{\{z \leq y\}} \big)^2 \dee y$. Therefore, we may update our estimation $\ffunc^{\fparams}$ for the conditional \CDF{} $F_{\policyparams,\timeidx}(z|\state,\action)$ by using the fact that
\begin{equation}
    F_{\policyparams,\timeidx}(\cdot|\state,\action) = \argmin_{F \in \FF} \,
    \EE_{\state_{\timeidx+1}^{\policyparams}\sim\PP} \bigg[ \int_{\Reals} \big(F(y | \state_{\timeidx},\action_{\timeidx}) - \Ind_{\{ \rv_{\timeidx}^{\policyparams} \leq y \}} \big)^2 \dee y \biggm| \state_{\timeidx}=\state, \action_{\timeidx}=\action \bigg].\label{eq:cond-CDF-elicitable}
\end{equation}
For each epoch of the training procedure, we first compute realizations of the costs-to-go $\rv_{\timeidx,b}^{\policyparams} := \costfunc_{\timeidx,b} + \qfunc^{\qparams}(\state_{\timeidx+1,b}, \policy^{\policyparams}(\state_{\timeidx+1,b}))$.
We then evaluate the integral within the continuous ranked probability score of \cref{eq:cond-CDF-elicitable} over a partition of $N$ points covering the cost space $\costspace$, denoted by $\{y_i\}_{i=1,\ldots,N}$.
The choice of the partition is ultimately application-dependent, and must be carefully chosen by the user.
The loss we aim to minimize is given by
\begin{equation}
    \Ll^{\fparams} = \frac{1}{B\timelength} \sum_{b=1}^{B} \sum_{\timeidx\in\timespace} \sum_{i=1}^{N} \Big( \ffunc^{\fparams}(\state_{\timeidx,b},\action_{\timeidx,b},y_i) - \Ind_{\{ \rv^{\policyparams}_{\timeidx,b} \leq y_i \}} \Big)^2 \Delta y_i. \label{eq:loss-adversary} \tag{L1}
\end{equation}
We repeat these steps to update the parameters $\fparams$ and train the \ANN{} $\ffunc^{\fparams}$ until convergence.

\subsection{First Moment of Costs-to-go and Q-function}
\label{ssec:critic}

For the critic, we want to estimate the first moment $\mufunc(\state,\action)$ and Q-function $\qfunc_{\timeidx}(\state,\action;\policyparams)$ while fixing the other \ANN{} structures.
Without loss of generality, we assume that the one-step conditional risk measures are 1-elicitable with a corresponding strictly consistent scoring function $\score$. Recall that, in our setting, the sequence of costs induced by the agent's policy is explained by the sequence of states and actions.
Therefore, the first moment and Q-function must be approximated using a function, as opposed to a point forecast, because it is an elicitable functional of the conditional \CDF{} given $(\state_{\timeidx},\action_{\timeidx})$.
We wish to find mappings $\mufunc,\qfunc:\statespace\times\actionspace \rightarrow \bar{\Reals}$ that minimize the following expected scores:
\begin{equation}
\begin{split}
    \min_{\mufunc:\statespace\times\actionspace \rightarrow \Reals} \, &\EE_{\state_{\timeidx+1}^{\policyparams}\sim\PP} \Big[ \big( \mufunc(\state_{\timeidx},\action_{\timeidx}) -\rv_{\timeidx}^{\policyparams} \big)^2 \Bigm| \state_{\timeidx}=\state, \action_{\timeidx}=\action \Big] \quad \text{and} \\
    \min_{\qfunc:\statespace\times\actionspace \rightarrow \Reals} \, &\EE_{\state_{\timeidx+1}^{\policyparams}\sim\PP} \Big[ \score\big( \qfunc(\state_{\timeidx},\action_{\timeidx}); \, \rv_{\timeidx}^{\adv} \big) \Bigm| \state_{\timeidx}=\state, \action_{\timeidx}=\action \Big].
\end{split}
\end{equation}
For the critic, we use this deep composite regression approach, where we restrict the space of mappings to \ANN{}s parametrized by some parameters $\muparams,\qparams$, respectively.
To keep a model-agnostic algorithm, we replace the expectation by the empirical mean over a batch of $B$ observed transitions.
Finally, we generate realizations of the worst-case costs-to-go $\rv_{\timeidx}^{\adv}$ using the inverse transform sampling on its optimal quantile function.

When estimating the first moment, we optimize the loss function
\begin{equation}
    \Ll^{\muparams} = \frac{1}{B\timelength} \sum_{\timeidx\in\timespace} \sum_{b=1}^{B} \Big( \mufunc^{\muparams}(\state_{\timeidx,b},\action_{\timeidx,b}) -  \rv_{\timeidx,b}^{\policyparams} \Big)^2.
    \label{eq:loss-mu} \tag{L2}
\end{equation}

Let $\tilde\policyparams,\tilde\qparams,\tilde\muparams,\tilde\fparams$ denote frozen parametrizations of the \ANN{} approximations that are slowly updated. This technique can be thought of as using target networks to avoid numerical instabilities \citep[see e.g.][]{van2016deep}.
For the Q-function, if the uncertainty set is of the form in \cref{eq:U-set-wass-mom} with a nondecreasing distortion function $\gamma_{\state}$, we may use the optimal quantile function $\breve{F}^{*}_{\adv}$ as in \cref{thm:optimal-quantile-fn-mom}. It leads to the loss function
\begin{equation}
    \Ll^{\qparams} = \frac{1}{B\timelength} \sum_{\timeidx\in\timespace} \sum_{b=1}^{B} \score\Bigg( \qfunc^{\qparams}(\state_{\timeidx,b},\action_{\timeidx,b}); \; \widetilde{\mufunc}_{\timeidx,b} + \frac{\lambda^{*}_{\timeidx,b} (\widetilde\rv_{\timeidx,b}^{\policyparams} - \widetilde{\mufunc}_{\timeidx,b}) + \gamma_{\state_{\timeidx,b}}( \widetilde{U}_{\timeidx,b}) - 1}{b_{\lambda^{*}_{\timeidx,b}}} \Bigg),
    \label{eq:loss-critic} \tag{L3}
\end{equation}
where $\widetilde{\mufunc}_{\timeidx,b} := \mufunc^{\tilde\muparams}(\state_{\timeidx,b},\action_{\timeidx,b})$, $\widetilde{\rv}_{\timeidx,b}^{\policyparams} := \costfunc_{\timeidx,b} + \qfunc^{\tilde\qparams}(\state_{\timeidx+1,b}, \policy^{\tilde\policyparams}(\state_{\timeidx+1,b}))$, $\widetilde{U}_{\timeidx,b} := \ffunc^{\tilde\fparams}(\state_{\timeidx,b}, \action_{\timeidx,b}, \widetilde\rv_{\timeidx,b}^{\policyparams})$, and $\lambda^{*}_{\timeidx,b}, b_{\lambda^{*}_{\timeidx,b}}$ for each transition are given in \cref{thm:optimal-quantile-fn-mom}.
Using target networks ensures that the random variables $\widetilde{U}_{\timeidx,b}$ remain conditionally uniform on $\Ff_{\timeidx}$ while the Q-function changes.

\begin{remark}
    Alternatively, if the uncertainty set is of the form in \cref{eq:U-set-wass}, we obtain an estimate of the appropriate optimal quantile function by (i) evaluating the \CDF{} $\ffunc^{\fparams}(\state_{t}, \action_{t}, \cdot)$ on the partition covering the cost space $\{y_i\}_{i}$, (ii) inverting it to get an estimate of the quantile function $\breve{\ffunc}^{\fparams}(\state_{t}, \action_{t}, \cdot)$, and (iii) performing an isotonic regression with a given $\lambda^{*}$ until the Wasserstein constraint in \cref{thm:optimal-quantile-fn} is satisfied. Altogether we wish to minimize the loss
    \begin{equation*}
        \Ll^{\qparams} = \frac{1}{B\timelength} \sum_{\timeidx\in\timespace} \sum_{b=1}^{B} \score\bigg( \qfunc^{\qparams}(\state_{\timeidx,b},\action_{\timeidx,b}); \;
        \Big( \breve{F}^{\fparams}(\state_{\timeidx,b}, \action_{\timeidx,b}, x) + \frac{\gamma_{\state_{\timeidx,b}}(x)}{2\lambda^{*}_{\timeidx,b}} \Big)^{\uparrow}\Big|_{x=\widetilde{U}_{\timeidx,b}} \bigg).
    \end{equation*}
    Note that this loss function involves constantly solving optimization problems to obtain $\lambda_{\timeidx}^{*}$ and calculating isotonic projections, which is computationally expensive.
\end{remark}

This approach is straightforward to implement with a 1-elicitable functional, as one substitutes $\score$ with the strictly consistent scoring function for the corresponding one-step conditional risk measure.
For $k$-elicitable mappings, one must modify the structure of $\qfunc^{\qparams}$ according to the number of elicitable mappings such that the one-step conditional risk measure becomes elicitable.
If we consider the dynamic $\CVaR_{\alpha}$, the Q-function consists of two \ANN{}s returning the approximations of the dynamic $\VaR_{\alpha}$ and the excess between the dynamic $\CVaR_{\alpha}$ and dynamic $\VaR_{\alpha}$, while the scoring function is of the form given in \cref{eq:score-CVaR}.
For general $\alpha$-$\beta$ risk measures with a distortion function characterized by
\begin{equation*}
    \gamma(u) = \frac{1}{\eta} \Big( p \Ind_{\{ u < \alpha \}} + (1-p) \Ind_{\{ u \geq \beta \}} \Big),
\end{equation*}
where $p\in [0,1]$, $0 < \alpha \leq \beta < 1$ and $\eta = p \alpha + (1-p) (1-\beta)$, we wish to estimate the vector $(\text{LTE}_{\alpha}, \VaR_{\alpha}, \VaR_{\beta}, \CVaR_{\beta})$,
which is 4-elicitable, and we recover the Q-function with $\frac{p \alpha}{\eta} \text{LTE}_{\alpha} + \frac{(1-p)(1-\beta)}{\eta} \CVaR_{\beta}$.
We refer the reader to \cite{coache2023conditionally} for a similar procedure with dynamic spectral risk measures.

\subsection{Policy}
\label{ssec:actor}

To update the policy, we want to update the policy parameters in the direction of the gradient of the value function.
In this procedure, we suppose that the Q-function $\qfunc^{\qparams}$, first moment $\mufunc^{\muparams}$ and \CDF{} $\ffunc^{\fparams}$ are fixed.

Using a mini-batch of $B$ full trajectories, we want to estimate the gradient in \cref{thm:gradient-V} and use it in a policy gradient approach.
The existence of this gradient requires the policy $\policy^{\policyparams}$ to satisfy some regularity assumptions regarding its continuity, which are fairly standard in \RL{} with neural networks as function approximators.
Altogether we aim to minimize the loss
\begin{equation}
\begin{split}
    \Ll^{\policyparams} &= \frac{1}{B\timelength} \sum_{\timeidx\in\timespace} \sum_{b=1}^{B} \grad{\policyparams} \policy^{\policyparams}(\state_{\timeidx,b}) \Bigg[ \grad{a} \qfunc^{\qparams}(\state_{\timeidx,b}, a; \policyparams)\Big|_{a = \policy^{\policyparams}(\state_{\timeidx,b})} \\
    &\qquad - \frac{b_{\lambda^{*}_{\timeidx,b}} - \lambda^{*}_{\timeidx,b}}{b_{\lambda^{*}_{\timeidx,b}}} \Big((b_{\lambda^{*}_{\timeidx,b}} - \lambda^{*}_{\timeidx,b}) (\rv_{\timeidx,b}^{\policyparams} - \mufunc_{\timeidx,b}) + 1\Big) \frac{ \grad{a} \ffunc^{\fparams}(\state_{\timeidx,b}, a, x)}{\grad{x} \ffunc^{\fparams}(\state_{\timeidx,b}, a, x)} \bigg|_{(x,a)=(\rv_{\timeidx,b}^{\policyparams},\policy^{\policyparams}(\state_{\timeidx,b}))} \Bigg].
\end{split}\label{eq:loss-actor} \tag{L4}
\end{equation}
We ignore any gradient $\grad{\rv^{\policyparams}_{\timeidx}} F_{\rv^{\policyparams}_{\timeidx}}$, because we fix $\ffunc^{\fparams}$ while performing a policy gradient step during the actor, i.e. the \ANN{}(s) do not explicitly depend on $\policyparams$.
That is a common approach in the literature \citep[see e.g.][]{degris2012off}.
We repeat these steps for a certain number of epochs, which updates the policy parameters in the direction of the gradient of the value function.
The number of epochs for the actor must remain relatively small, because the estimations $\qfunc^{\qparams},\mufunc^{\muparams},\ffunc^{\fparams}$ quickly become obsolete as the policy changes.

\section{Universal Approximation Theorem}
\label{sec:UAT-results}
In this section, we show that, all things being held equal, there exist sufficiently large \ANN{}s of the form given in \cref{sec:algorithm} accurately approximating the relevant mappings, in the same spirit as universal approximation theorems.
These theorems rely on the universal approximation theorem for arbitrary width (\cref{thm:UAT}), the fact that finite ensembles of \ANN{}s can be approximated by a single \ANN{} with augmented input space (\cref{thm:ensemble}), and the robustification considered here preserves cash additivity and monotonicity (\cref{thm:monetary}), given in the supplemental materials for completeness.

\begin{theorem}
    Let $\policy^{\policyparams},\ffunc^{\fparams},\qfunc^{\qparams}$ be fixed. Then, for any $\varepsilon>0$, there exists an \ANN{}, denoted by $\mufunc^{\muparams}$, such that $\forall \, \timeidx\in\timespace$,
    \begin{equation*}
        \sup_{(\state,\action)\in\statespace\times\actionspace} \Big\Vert \EE\Big[ \costfunc_{\timeidx}+\qfunc^{\qparams}(\state_{\timeidx+1}, \policy^{\policyparams}(\state_{\timeidx+1})) \Bigm| (\state_{\timeidx},\action_{\timeidx}) = (\state,\action)\Big] - \mufunc^{\muparams}(\state,\action) \Big\Vert < \varepsilon.
    \end{equation*}
\end{theorem}
\begin{proof}
    We give a sketch of the proof, as the result follows closely Theorem 6.2 of \cite{coache2024reinforcement}. We aim to estimate the dynamic expectation of the costs-to-go, where each one-step conditional risk measure is monetary. Mappings satisfying cash additivity and monotonicity are Lipschitz continuous and, hence, absolutely continuous.
    Using \cref{thm:UAT}, for each $\timeidx\in\timespace$, there exists an \ANN{}, denoted by $\mufunc_{\timeidx}$, approximating to an arbitrary accuracy $\varepsilon_{\timeidx}>0$ the first moment of the costs-to-go $\EE[ \costfunc_{\timeidx}+\qfunc^{\qparams}(\state_{\timeidx+1}, \policy^{\policyparams}(\state_{\timeidx+1}))| \state_{\timeidx},\action_{\timeidx}]$. Using \cref{thm:ensemble}, there exists a single \ANN{} approximating to an arbitrary accuracy this collection of \ANN{}s $\{\mufunc_{\timeidx}\}_{\timeidx\in\timespace}$.
\end{proof}

\begin{theorem}
    Let $\policy^{\policyparams},\ffunc^{\fparams},\mufunc^{\muparams}$ be fixed and consider the Q-function in \cref{eq:q-function-algo}. Then, for any $\varepsilon>0$, there exists an \ANN{}, denoted by $\qfunc^{\qparams}$, such that, $\forall \, \timeidx\in\timespace$,
    \begin{equation*}
        \sup_{(\state,\action)\in\statespace\times\actionspace} \Big\Vert \qfunc_{\timeidx}(\state, \action;\policyparams) - \qfunc^{\qparams}(\state,\action) \Big\Vert < \varepsilon.
    \end{equation*}
\end{theorem}
\begin{proof}
    Again, we give a sketch of the proof, as the result follows closely Theorem 6.1 and Corollary 6.2 of \cite{coache2023conditionally}. Assume that the Q-function is $k$-elicitable with a given decomposition. Using \cref{thm:monetary}, the one-step conditional risk measures, which are robust distortion risk, satisfy the monetary properties and, thus, are absolutely continuous. Furthermore, all $k$ components may be expressed as linear combinations of $\VaR$s and $\CVaR$s, so they are absolutely continuous as well.

    We use a proof by induction to show that all each component may be approximated arbitrarily accurately at every period $\timeidx\in\timespace$ as long as we have an adequate approximation at the subsequent periods. Using \cref{thm:UAT}, the base case at $\timeidx=\timelength$ is true. We then assume that the statement is true for $\timeidx+1$, that is there exist \ANN{}s, denoted by $\qfunc_{\kappa,\timeidx}$, approximating to an arbitrary accuracy $\varepsilon_{\kappa,\timeidx}>0$ the $\kappa$-th component for $\kappa=1,\ldots,k$. We prove that the $k$ components may be approximated to any arbitrary accuracy at $\timeidx$ using the cash additivity property, triangle inequality and \cref{thm:UAT}, which completes the proof by induction.

    Using \cref{thm:ensemble}, there exist $k$ \ANN{}s approximating to an arbitrary accuracy the collection of \ANN{}s $\{\qfunc_{\kappa,\timeidx}\}_{\timeidx\in\timespace}$ for $\kappa=1,\ldots,k$. Finally, we can construct a single \ANN{} approximating to an arbitrary accuracy the Q-function using the triangle inequality, since it is a linear combination of those $k$ \ANN{}s.
\end{proof}

\begin{remark}
    Regarding a universal approximation theorem for the conditional \CDF{}, the issue here lies in the partial monotonicity of the conditional \CDF{}, which makes proving uniform convergence results quite challenging. Nonetheless, we believe the combination of unconstrained and constrained layers preserves some suitable universal approximation guarantees. Indeed, we suspect that results such as Theorem 2 of \cite{mikulincer2022size} may be extended to our setting, potentially with an accuracy depending on the nonmonotonic inputs, but this remains to be proven. In practice, as we are more invested in ensuring the monotonicity property to avoid numerical instabilities than the approximation guarantees, we recover an adequate estimation of the costs-to-go distribution as part of our actor-critic algorithm, which points that this conjecture holds given a sufficiently large \ANN{} in terms of width and depth.
\end{remark}

\section{Experimental Results}
\label{sec:experiments}

This collection of experiments is performed on a portfolio allocation problem similar to Section 7.2 of \cite{coache2023conditionally}.
Suppose an agent can allocate her wealth between different risky assets during $\timelength=12$ periods over a six month horizon.
The agent intents on minimizing a dynamic risk measure of her profit and loss (\PnL{}) with robust distortion one-step conditional risk measures.
Depending on her own risk preferences, the agent has the possibility to tune her objective by (i) selecting a distortion function that leads to risk-neutral, risk-averse or risk-seeking policies, and (ii) choosing larger $\epsilon$'s to robustify her actions against the uncertainty of the true dynamics of the market.

The choice of the tolerance $\epsilon$ may be influenced by some additional exogenous information, such as expert opinion or known bounds on the moments of the cost distribution that the agent is willing to accept \citep[see e.g.][]{pesenti2023portfolio,bernard2024robust}.
Otherwise, the agent may decide on an appropriate $\epsilon$ using the following data-driven approach.
Given a training environment, suppose that the agent designs a testing environment she aims to robustify against.
The tolerance $\epsilon$ should then be just large enough such that the optimal time-consistent robust policy (i) stays close to the optimal policy of the training environment, and (ii) performs relatively well for the testing environment.
From the agent's perspective, this specific choice of dynamic robust distortion risk measure gives a notion of robustness she is willing to tolerate for a given pair of training-testing environments, which should robustify as well other (unknown) testing environments.

% \begin{table}[htbp]
\begin{wraptable}{r}{0.5\textwidth}
    \centering
    \footnotesize
    \vspace{1ex}\centering
    \begin{tabular}{c r r r} 
    \toprule\toprule
    Stock & $\price_0$ & Mean & Std. dev. \\ 
    \midrule
    \rowcolor{mred!30} AAL & 31.76 & 0.0137 & 0.1696 
    \\
    \rowcolor{mblue!30} AMZN & 1780.75 & 0.0025 & 0.0707 
    \\
    \rowcolor{mred!30} CCL & 50.72 & 0.0217 & 0.2145
    \\
    \rowcolor{mblue!30} FB & 166.69 & 0.0031 & 0.0784
    \\
    \rowcolor{mblue!30} IBM & 141.10 & 0.0026 & 0.0732 
    \\
    \rowcolor{mblue!30} INTC & 53.7 & 0.0044 & 0.0947
    \\
    \rowcolor{mred!30} LYFT & 78.29 & 0.0130 & 0.1662
    \\
    \rowcolor{mred!30} OXY & 66.20 & 0.0193 & 0.2023
    \\
    \bottomrule\bottomrule
    \end{tabular}
	\caption{Initial price $\price_0$, mean and standard deviation of the relative price change $\frac{\price_{\timeidx+1} - \price_{\timeidx}}{\price_{\timeidx}}$ for each asset estimated over 100,000 simulated sample paths from the \VECM{}. We highlight two distinct groups: riskier assets with greater returns on average (in blue) and less volatile assets with smaller returns (in orange).}
	\label{fig:vecm-stats}
\end{wraptable}
% \end{table}
We consider price dynamics driven by a co-integration model to mimic realistic price paths.
More precisely, we estimate a vector error correction model (\VECM{}) using daily data from $I = 8$ different stocks listed on the NASDAQ exchange between September 31, 2020 and December 31, 2021 inclusively.
The resulting estimated model, with two cointegration factors and no lag differences (both selected using the BIC criterion), is used as a simulation engine to generate price paths $(\price_{\timeidx}^{(i)})_{\timeidx}$, $i\in\Ii := \{1,\ldots,I\}$. We refer the reader to Appendix C of \cite{coache2023conditionally} for explanations on \VECM{}s and the parameter estimates for this dataset.
We report some statistics of interest for the different stocks in \cref{fig:vecm-stats}.

At each period $\timeidx \in \timespace$, the agent observes the information available and decides on the proportions of her wealth, denoted by $(\policy_{\timeidx}^{(i)})_{\timeidx}$, $i\in\Ii$, to invest in the different financial instruments.
We impose these actions to be a $I$-simplex in order to avoid short selling, i.e. $\policy_{\timeidx}^{(i)} \geq 0$, $\forall i\in\Ii$ and $\sum_{i\in\Ii} \policy_{\timeidx}^{(i)} = 1$, by applying a softmax transformation on the output layer of the policy network.
The observed costs corresponds to the \PnL{} fluctuation between two subsequent periods $\costfunc_{\timeidx} = y_{\timeidx} + y_{\timeidx+1}$, where the agent's wealth $(y_{\timeidx})_{\timeidx}$ is determined by
\begin{equation}
	\dee y_{\timeidx} = y_{\timeidx} \Bigg( \sum_{i=1}^{I} \policy_{\timeidx}^{(i)} \frac{\dee \price_{\timeidx}^{(i)}}{\price_{\timeidx}^{(i)}} \Bigg), \quad y_{0} = 1.
\end{equation}
We suggest the following approach to include exploratory noise in the current policy. For each action $a = \policy(\state) \in [0,1]^{I}$, we generate independently a Bernoulli random variable $\xi \in \{0,1\}$ with an exploration probability $p_{\text{ex}} = 0.5$ such that $\xi \sim \bernoullidist(p_{\text{ex}})$.
For each action such that $\xi=1$, we additionally generate a Dirichlet noise $\Nn \sim \text{Dirichlet}(0.05)$, and then compute the randomized action $a' = 0.75 a + 0.25 \Nn$. Algorithm parameters and computation times are included in \cref{sec:appendix-params}.

To understand what the learned optimal policy dictates, \cref{fig:CVaR01} shows the average investment proportions in each asset for every period when optimizing a dynamic robust $\CVaR_{0.1}$ and varying the the uncertainty tolerance $\epsilon$. Without robustification, when $\epsilon=0$, the agent prioritizes a mix of assets, some with greater returns and others with lower volatilities. The learned optimal policy evolves as we increase $\epsilon$ and the \PnL{} distribution moves to the left. For larger tolerances, the learned policy becomes closer to an equal weight portfolio over all available assets.
We then repeat the exercise with the dynamic robust $\CVaR_{0.2}$, as shown in \cref{fig:CVaR02}, and obtain the same behavior when increasing $\epsilon$. Additionally, with a higher risk-awareness, we observe that the learned optimal policy prefers investments in less volatile assets on average. In general, larger uncertainty tolerances lead to more conservative policies that do not perform optimally in terms of \PnL{}, and we retrieve the non-robust optimal policies as $\epsilon$ decreases to zero.

\begin{figure}[htbp]
    \centering
    \begin{minipage}[b]{0.96\textwidth}
        \subfloat[Learnt policies]{\label{fig:CVaR01-policy} \includegraphics[width=0.95\textwidth]{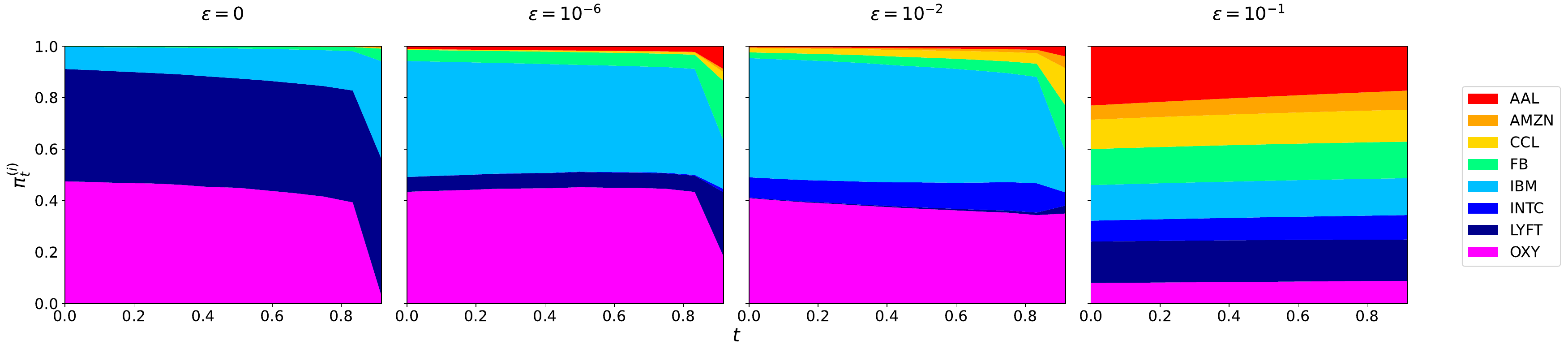}}
    \end{minipage}
    \hfill
    \begin{minipage}[b]{0.60\textwidth}
        \subfloat[Distributions of terminal \PnL{}]{\label{fig:CVaR01-pnl} \includegraphics[width=0.95\textwidth]{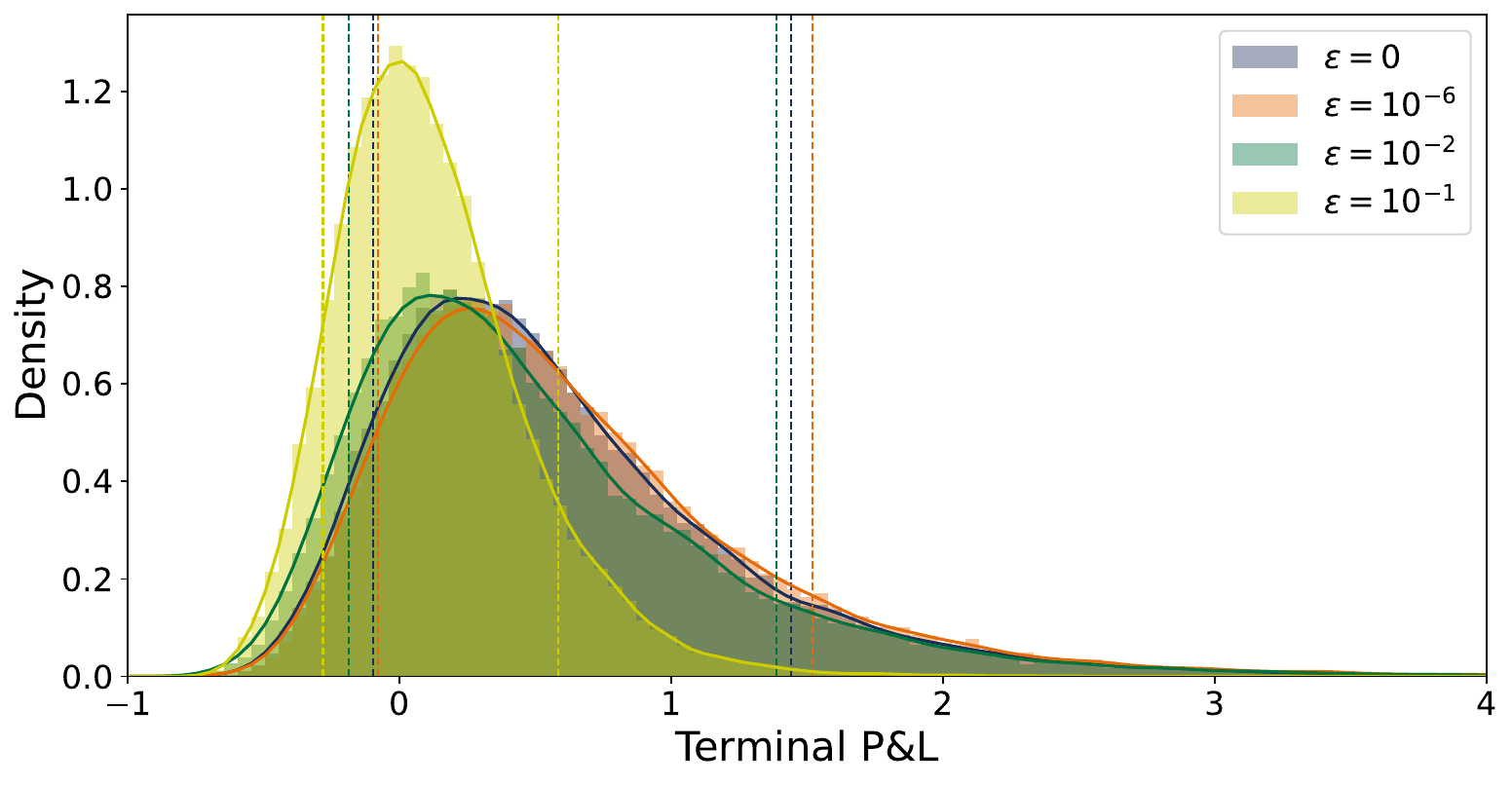}}
    \end{minipage}
	\caption{\PnL{} distributions when following learnt optimal policies wrt dynamic robust $\CVaR_{0.1}$.}
	\label{fig:CVaR01}
\end{figure}
\begin{figure}[htbp]
    \centering
    \begin{minipage}[b]{0.96\textwidth}
        \subfloat[Learnt policies]{\label{fig:CVaR02-policy} \includegraphics[width=0.95\textwidth]{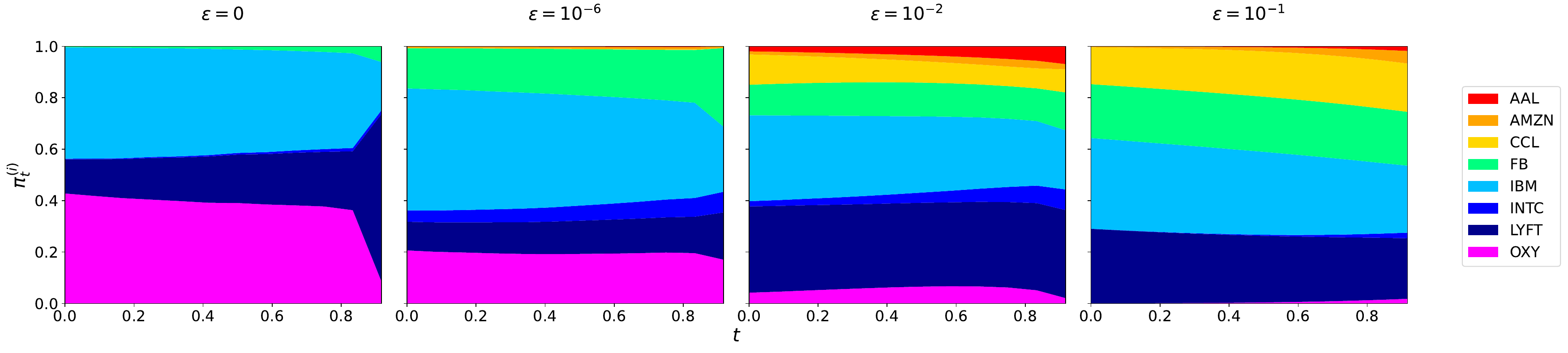}}
    \end{minipage}
    \hfill
    \begin{minipage}[b]{0.60\textwidth}
        \subfloat[Distributions of terminal \PnL{}]{\label{fig:CVaR02-pnl} \includegraphics[width=0.95\textwidth]{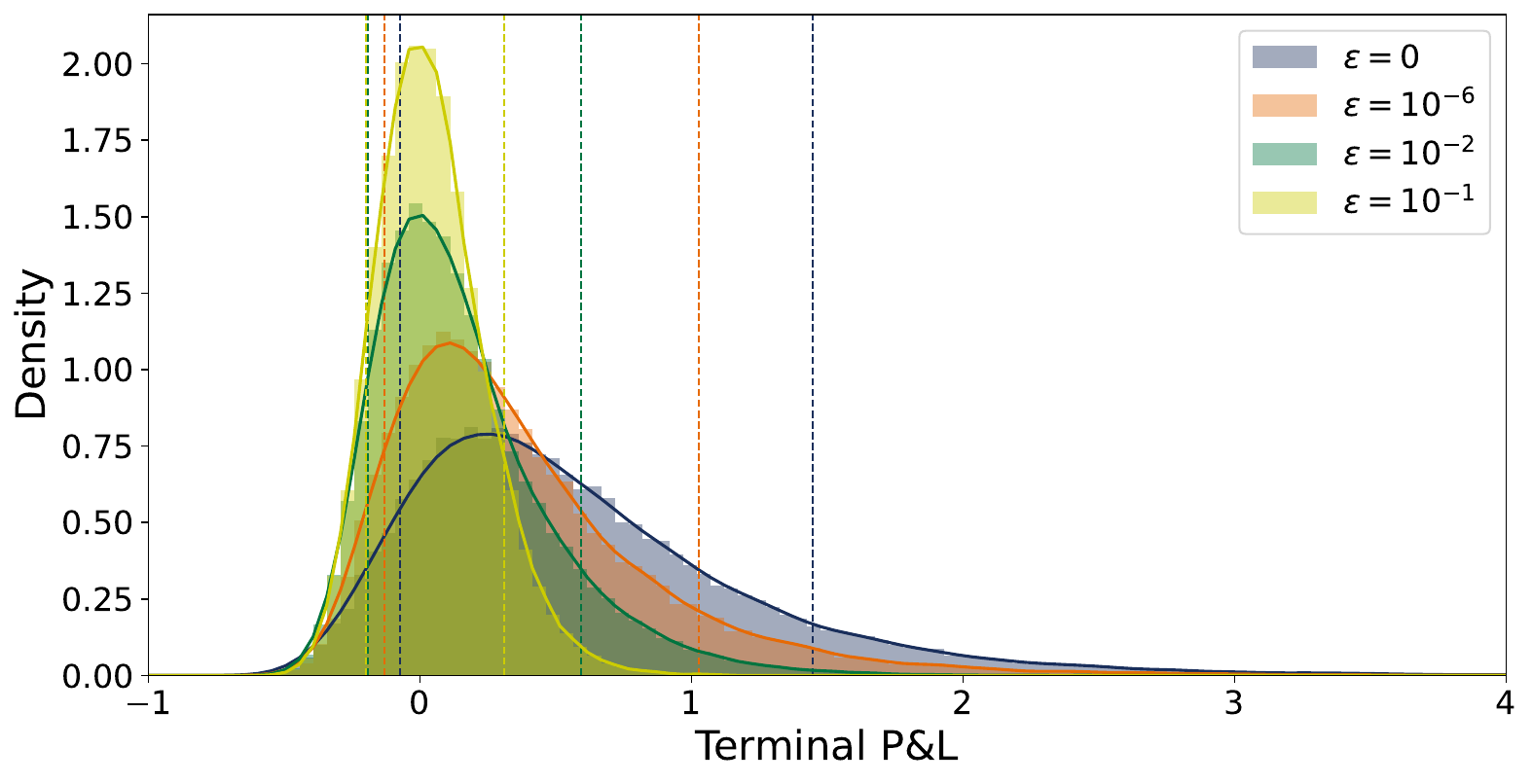}}
    \end{minipage}
	\caption{\PnL{} distributions when following learnt optimal policies wrt dynamic robust $\CVaR_{0.2}$.}
	\label{fig:CVaR02}
\end{figure}

\section{Conclusion}
\label{sec:conclusion}

In this work, we present a robust \RL{} framework for time-consistent risk-aware agents.
Our approach utilizes dynamic risk measures constructed with a class of robust distortion risk measures of all random variables within a Wasserstein uncertainty set to robustify the agent's actions.
We estimate the dynamic risk using the elicitability of distortion risk measures and derive a deterministic policy gradient procedure by reformulating the optimization problem via a quantile representation.
Furthermore, we show that our proposed deep learning algorithm performs well on a portfolio allocation example.

One limitation of the universal approximation theorems provided in \cref{sec:UAT-results} is that while we prove the existence of arbitrarily accurate \ANN{}s, we do not show how to attain them. It remains an open challenge to provide methodologies with convergence guarantees to the true dynamic risk, and, more generally, to develop deep actor-critic algorithms with convergence guarantees to the optimal policy.
Other interesting avenues to pursue in future work consist of deriving \RL{} algorithms for different classes of time-consistent dynamic robust risk measures, such as those with an uncertainty set constructed using the Kullback-Leibler divergence, and formally comparing robust \RL{} approaches available in the literature.

% references
\bibliographystyle{plainnat}
\bibliography{references}

\appendix

\section{Algorithm Hyperparameters}
\label{sec:appendix-params}

From an algorithmic perspective, in \cref{sec:experiments}, we use $K^{\fparams}=5$ epochs for updating the CDF of costs-to-go, $K^{\qparams}=5$ epochs for updating the Q-function, $K^{\muparams}=5$ epochs for updating the expected costs-to-go, and $K^{\policyparams}=1$ epoch for updating the policy.
We generate a partition $\{y_i\}_{i}$ of 501 evenly spaced points on an interval that covers the costs-to-go $\widetilde{\rv}_{\timeidx,b}^{\policyparams}$ at each iteration, and use mini-batch sizes of $B^{\fparams}=B^{\qparams}=B^{\muparams}=B^{\policyparams}=128$.
The \ANN{}s have the following structure:
\begin{enumerate}
    \item[$\policy^{\policyparams}$]:  five layers of 32 hidden nodes each with SiLU activation functions and a softmax output activation function;
    \item[$\qfunc^{\qparams}$]:  six layers of 32 hidden nodes each with SiLU activation functions, a tanh output activation function for the dynamic $\VaR$ and a softplus output activation function for the excess between dynamic $\CVaR$ and $\VaR$;
    \item[$\mufunc^{\muparams}$]:  six layers of 32 hidden nodes each with SiLU activation functions;
    \item[$\ffunc^{\fparams}$]:  four layers of 32 hidden nodes each with SiLU activation functions, followed by four layers of 32 hidden nodes each with tanh activation functions, positive weights, and a sigmoid output activation function.
\end{enumerate}
Learning rates for the critic are initially set to $5\times 10^{-4}$, while the learning rate for the actor is set to $3\times 10^{-6}$ and decays by 0.999995 at every epoch.
The target networks are updated at every iteration with a soft update parameter of $\tau=0.008$.
We train all models for 500,000 iterations (approximately 24 hours) on the Graham servers, managed by the \href{https://alliancecan.ca/en}{Digital Research Alliance of Canada}.

To analyze the computational speed of our algorithm, we report the execution time per 10 iterations using a Tesla T4 GPU on \href{https://colab.google/}{Google Colaboratory} in \cref{fig:comp-times}. Again, each iteration consists of $K^{\fparams}=K^{\qparams}=K^{\muparams}=5$ epochs for the critic, before executing the actor for $K^{\policyparams}=1$ epoch. Our results seems to indicate sublinear scaling when increasing the number of periods and mini-batch sizes, while the number of assets has a negligible effect. It is important to note here that (i) achieving high precision requires more iterations when increasing the number of periods (see the universal approximation theorems in \cref{sec:UAT-results}), and (ii) the size of GPU memory for computing gradients represents the main computational bottleneck.

\begin{table}[htbp]
    \begin{center}
    \hfill
    \begin{minipage}[b]{0.47\textwidth}
        \subfloat[4 assets]{\vspace{1ex}\centering
        \begin{tabular}{c c c c} 
        \toprule\toprule
        & $B=128$ & $B=256$ & $B=512$ \\ 
        \midrule
        $T=6$ & 1.64 & 2.14 & 3.26
        \\
        $T=8$ & 1.69 & 2.32 & 3.55
        \\
        $T=12$ & 1.90 & 2.66 & 4.15
        \\
        $T=24$ & 2.68 & 3.78 & 5.95
        \\
        $T=48$ & 4.20 & 6.01 & 8.54
        \\
        \bottomrule\bottomrule
        \end{tabular}}
    \end{minipage}
    \hfill
    \begin{minipage}[b]{0.47\textwidth}
        \subfloat[8 assets]{\vspace{1ex}\centering
        \begin{tabular}{c c c c} 
        \toprule\toprule
        & $B=128$ & $B=256$ & $B=512$ \\ 
        \midrule
        $T=6$ & 1.61 & 2.21 & 3.31
        \\
        $T=8$ & 1.77 & 2.36 & 3.63
        \\
        $T=12$ & 2.02 & 2.73 & 4.24
        \\
        $T=24$ & 2.80 & 3.91 & 6.10
        \\
        $T=48$ & 4.49 & 6.23 & 8.93
        \\
        \bottomrule\bottomrule
        \end{tabular}}
    \end{minipage}
    \hfill
    \end{center}
	\caption{Average execution time, in seconds, for 10 iterations of the actor-critic algorithm estimated over 10 runs.}
	\label{fig:comp-times}
\end{table}

\section{Additional Lemmas}

\begin{lemma}{(e.g., Theorem 3.1 of \cite{pinkus1999approximation})}
    \label{thm:UAT}
    Let $d\in\Nats$ and $K\subset\Reals^{d}$ be a compact subset. For any $\varepsilon>0$ and continuous function $f(x)$, $x\in K$, there exists an \ANN{}, denoted by $\hat{f}$, such that $\sup_{x\in K} \Vert f(x) - \hat{f}(x) \Vert < \varepsilon$ if and only if the activation function is not a polynomial.
\end{lemma}

\begin{lemma}{(Lemma 6.4 of \cite{coache2024reinforcement})}
    \label{thm:ensemble}
    Let $d\in\Nats$ and $K\subset\Reals^{d}$ be a compact subset. For any $\varepsilon>0$ and ensemble of a finite number of \ANN{}s $\{\hat{f}_{\timeidx}(x)\}_{\timeidx\in\timespace}$, $x\in K$, there exists an \ANN{}, denoted by $\hat{g}$, such that $\sup_{x\in K} \Vert \hat{f}_{\timeidx}(x) - \hat{g}_{\timeidx}(x) \Vert < \varepsilon$, $\forall \, \timeidx\in\timespace$.
\end{lemma}

\begin{lemma}
    \label{thm:monetary}
    Let $\riskmeas_{\timeidx}$ be a monetary one-step conditional risk measure -- that is cash additive, where $\riskmeas_{\timeidx}(m+\rv) = m+\riskmeas_{\timeidx}(\rv)$ for any $m \in \Lpspace_{\timeidx}$, and monotone, where $\rv \leq \rvdum$ implies $\riskmeas_{\timeidx}(\rv) \leq \riskmeas_{\timeidx}(\rvdum)$. Consider the robust version $\robust_{\timeidx}^{\epsilon_{\state}}$ under the 2-Wasserstein distance with an uncertainty set of either the form \cref{eq:U-set-wass} or \cref{eq:U-set-wass-mom}. Then, $\robust_{\timeidx}^{\epsilon_{\state}}$ remains monetary.
\end{lemma}
\begin{proof}
    We have that $\robust_{\timeidx}^{\epsilon_{\state}}$ is cash additive since, for any $m \in \Lpspace_{\timeidx}$,
    \begin{equation*}
        \robust_{\timeidx}^{\epsilon_{\state}}(m+\rv)
        = \esssup_{\rv^{\adv} \in \varphi_{m+\rv}^{\epsilon_{\state}}} \riskmeas_{\timeidx}(\rv^{\adv})
        = \esssup_{\rv^{\adv} \in \varphi_{\rv}^{\epsilon_{\state}}} \riskmeas_{\timeidx}(m+\rv^{\adv})
        = m + \robust_{\timeidx}^{\epsilon_{\state}}(\rv).
    \end{equation*}
    Next, we show that $\robust_{\timeidx}^{\epsilon_{\state}}$ remains monotone by contradiction using the fact that the conditional 2-Wasserstein distance defines a metric on the space of probability measures. We define $\rv^{*} := \argmax_{\rv^{\adv} \in \varphi_{\rv}^{\epsilon_{\state}}} \riskmeas_{\timeidx}(\rv^{\adv})$ and denote the conditional quantile functions of $\rv,\rv^{*}$ and $\rvdum$ by respectively $\breve{F}_t,\breve{F}_{*,t}$ and $\breve{G}_t$. Let $\rv \leq \rvdum$ and assume that $\rv^{*} > \rvdum^{*}$. We cannot have $\langle \breve{F}_{*,t} , \breve{G_t} \rangle \leq \epsilon_{\state}$, because we would obtain $\rv^{*} \leq \rvdum^{*}$. On the other hand, if $\langle \breve{F}_{*,t} , \breve{G_t} \rangle > \epsilon_{\state}$, we get, using the triangle inequality, that
    \begin{equation*}
        \epsilon_{\state} < \langle \breve{F}_{*,t} , \breve{G}_{t} \rangle
        \leq \langle \breve{F}_{*,t} , \breve{F}_{t} \rangle + \langle \breve{F}_{t} , \breve{G}_{t} \rangle
        \leq \epsilon_{\state} + \langle \breve{F}_{t} , \breve{G}_{t} \rangle,
    \end{equation*}
    which leads to a contradiction, since $\langle \breve{F}_{t} , \breve{G}_{t} \rangle \geq 0$. Therefore, we must have $\rv^{*} \leq \rvdum^{*}$ and
    \begin{equation*}
        \robust_{\timeidx}^{\epsilon_{\state}}(\rv) = \riskmeas_{\timeidx}^{\epsilon_{\state}}(\rv^{*}) \leq \riskmeas_{\timeidx}^{\epsilon_{\state}}(\rvdum^{*}) = \robust_{\timeidx}^{\epsilon_{\state}}(\rvdum),
    \end{equation*}
    where the inequality follows from the monotonicity of $\riskmeas_{\timeidx}$.
\end{proof}

\section{Proofs}

\subsection{Proof of Theorem \ref{thm:optimal-quantile-fn}}
\label{sec:proof-optimal-quantile-fn}

\begin{proof}
    Since we are working with distortion risk measures and the Wasserstein distance, both components of the optimization problem in \cref{eq:q-function-algo} can be expressed in terms of quantile functions instead of random variables.
    Indeed, we have
    \begin{equation*}
        \varphi^{\epsilon_{\state}}_{\breve{F}_{\policyparams,\timeidx}(\cdot|\state,\action)} = \bigg\{ \breve{F} \in \breve{\FF} \, : \, \Big\Vert \breve{F}(\cdot|\state,\action) - \breve{F}_{\policyparams,\timeidx}(\cdot|\state,\action) \Big\Vert \leq \epsilon_{\state} \bigg\}
    \end{equation*}
    and the one-step conditional distortion risk measure $\big\langle \gamma_{\state} , \breve{F}_{\adv,\timeidx}(\cdot|\state,\action) \big\rangle$. Therefore, we have the equivalence relationship
    \begin{equation}
        \esssup_{\rv^{\adv}_{\timeidx} \in \varphi_{\rv_{\timeidx}^{\policyparams}}^{\epsilon_{\state}}} \; \riskmeas_{\timeidx}^{\gamma_{\state}} \Big( \rv_{\timeidx}^{\adv} \Bigm|\state_{\timeidx}=\state,\action_{\timeidx}=\action \Big)
        \equiv
        \esssup_{\breve{F}_{\adv,\timeidx} \in \varphi_{\breve{F}_{\policyparams,\timeidx}(\cdot|\state,\action)}^{\epsilon_{\state}}} \; \left\langle \gamma_{\state} , \breve{F}_{\adv,\timeidx}(\cdot|\state,\action) \right\rangle, \label{eq:quantile-reformulation}
    \end{equation}
    where the equivalence is to be understood as: (i) the quantile function of any optimal random variable for the left-hand side of \cref{eq:quantile-reformulation} is optimal for the right-hand side; and (ii) any random variable with an optimal quantile function for the right-hand side of \cref{eq:quantile-reformulation} is optimal for the left-hand side.
    Here, we remark that, as opposed to the original formulation of the problem, the optimization problem on the right-hand side is convex over the space of quantile functions.
    We use the Lagrange multiplier method to find the optimal solution.
    We have
    \begin{align}
        &L(\breve{F}_{\adv,\timeidx}, \lambda; \policyparams) \nonumber\\
        &\quad = \left\langle \gamma_{\state} , \breve{F}_{\adv,\timeidx}(\cdot|\state,\action) \right\rangle
        - \lambda \bigg( \left\Vert \breve{F}_{\adv,\timeidx}(\cdot|\state,\action) - \breve{F}_{\policyparams,\timeidx}(\cdot|\state,\action) \right\Vert^2 - \epsilon_{\state}^{2} \bigg) \nonumber\\
        \begin{split}
        \text{\scriptsize{[square completion]}}&\quad = -\lambda \left\Vert \breve{F}_{\adv,\timeidx}(\cdot|\state,\action) - \Big(\breve{F}_{\policyparams,\timeidx}(\cdot|\state,\action) + \frac{\gamma_{\state}}{2\lambda}\Big) \right\Vert^2 
        \\
        &\qquad\quad + \lambda \left( \epsilon_{\state}^{2} - \left\Vert \breve{F}_{\policyparams,\timeidx}(\cdot|\state,\action)\right\Vert^2 \right) + \frac{\left\Vert2\lambda\breve{F}_{\policyparams,\timeidx}(\cdot|\state,\action) + \gamma_{\state}\right\Vert^2}{4\lambda}.
        \end{split} \label{eq:lagrange-U-set-wass}
    \end{align}
    Using Slater's condition and the convexity of the quantile representation problem, strong duality holds:
    \begin{equation*}
        \qfunc_{\timeidx}(\state,\action;\policyparams) = \max_{\breve{F}_{\adv,\timeidx}\in\breve{\FF}} \min_{\lambda>0} L(\breve{F}_{\adv,\timeidx}, \lambda; \policyparams) = \min_{\lambda>0} \max_{\breve{F}_{\adv,\timeidx}\in\breve{\FF}} L(\breve{F}_{\adv,\timeidx}, \lambda; \policyparams).
    \end{equation*}
    Since only the first integral in \cref{eq:lagrange-U-set-wass} actually depends on $\breve{F}_{\adv}$, the inner optimization problem is attained for a given $\lambda$ by the isotonic projection
    \begin{equation*}
        \breve{F}^{*}_{\adv,\timeidx}(\cdot|\state,\action) = \bigg( \breve{F}_{\policyparams,\timeidx}(\cdot|\state,\action) + \frac{\gamma_{\state}(\cdot)}{2\lambda} \bigg)^{\uparrow}.
    \end{equation*}
    Finally, for the outer problem, the Wasserstein constraint is binding, and thus the optimal $\lambda^{*}$ is the positive value such that $\big\Vert \breve{F}^{*}_{\adv,\timeidx} - \breve{F}_{\policyparams,\timeidx} \big\Vert =\epsilon_{\state}$, which gives the desired result.
\end{proof}

\subsection{Proof of Theorem \ref{thm:optimal-quantile-fn-mom}}
\label{sec:proof-optimal-quantile-fn-mom}

\begin{proof}
    Similarly to \cref{thm:optimal-quantile-fn}, we use the Lagrange multiplier method to find the optimal solution. This differs from the original proof from \cite{bernard2024robust}, which uses properties of the covariance. In this proof, we remove the dependence on the time and state-action pair for readability. By square completion, we have
    \begin{align*}
        &L(\breve{F}_{\adv}, \lambda, \zeta, \eta; \policyparams) \\
        &\quad = \big\langle \gamma_{\state} , \breve{F}_{\adv} \big\rangle
        - \lambda \Big( \big\Vert \breve{F}_{\adv} - \breve{F}_{\policyparams} \big\Vert^2 - \epsilon_{\state}^{2} \Big) - \zeta \Big( \big\Vert \breve{F}_{\adv} \big\Vert^2 - \big\Vert \breve{F}_{\policyparams} \big\Vert^2 \Big) - \eta \Big( \big\langle \breve{F}_{\adv},1\big\rangle - \mu \Big) \\
        &\quad = -(\lambda+\zeta) \bigg\Vert \breve{F}_{\adv} - \frac{2 \lambda \breve{F}_{\policyparams} + \gamma_{\state} - \eta}{2(\lambda+\zeta)} \bigg\Vert^2 + (\zeta - \lambda) \big\Vert \breve{F}_{\policyparams} \big\Vert^2 + \lambda \epsilon_{\state}^2 + \eta \mu + \frac{\big\Vert 2 \lambda \breve{F}_{\policyparams} + \gamma_{\state} - \eta \big\Vert^2}{4 (\lambda + \zeta)}.
    \end{align*}
    Using Slater's condition and the convexity of the quantile representation problem, strong duality holds. The optimal quantile function thus has the following form:
    \begin{equation}
        \breve{F}^{*}_{\adv} = \frac{2 \lambda \breve{F}_{\policyparams} + \gamma_{\state} - \eta}{2(\lambda+\zeta)} = \frac{\lambda \breve{F}_{\policyparams} + \gamma_{\state} - a_{\lambda}}{b_{\lambda}}.
        \label{eq:a-b-lambda}
    \end{equation}
    From the Lagrangian constraint on the first moment, we get
    \begin{equation*}
        \big\langle \breve{F}^{*}_{\adv}, 1 \big\rangle = \mu
        \implies
        a_{\lambda} = (\lambda \mu + 1) - b_{\lambda}\mu.
    \end{equation*}
    From the Lagrangian constraint on the second moment, we obtain
    \begin{equation*}
    \begin{split}
        \big\Vert\breve{F}^{*}_{\adv} \big\Vert^2 = \big\Vert \breve{F}_{\policyparams} \big\Vert^2
        &\implies
        b_{\lambda}^2 \sigma^2 = \big\Vert \lambda \breve{F}_{\policyparams} + \gamma_{\state} \big\Vert^2 - (\lambda\mu + 1)^2 \\
        &\implies b_{\lambda} = \frac{\sqrt{ (\lambda \sigma)^2 + \sigma_{\gamma}^2 + 2\lambda \big(\langle \breve{F}_{\policyparams}, \gamma_{\state} \rangle - \mu\big) }}{\sigma}.
    \end{split}
    \end{equation*}
    Next, let $K = \sigma^2 - \frac{\epsilon_{\state}^2}{2}$. From the Lagrangian constraint on the Wasserstein distance, we get
    \begin{align}
        &\big\Vert \breve{F}^{*}_{\adv} - \breve{F}_{\policyparams} \big\Vert^2 = \epsilon^2 \nonumber\\
        &\implies \big\Vert \breve{F}^{*}_{\adv} \big\Vert^2 + \big\Vert \breve{F}_{\policyparams} \big\Vert^2 - 2 \big\langle  \breve{F}^{*}_{\adv} , \breve{F}_{\policyparams} \big\rangle = \epsilon^2 \nonumber\\
        &\implies K = \frac{\lambda \big\Vert \breve{F}_{\policyparams} \big\Vert^2 + \big\langle \breve{F}_{\policyparams}, \gamma_{\state} \big\rangle - a_{\lambda}\mu}{b_{\lambda}} - \mu^2 \nonumber\\
        &\implies b_{\lambda} K =  \lambda \sigma^2 + \big\langle \breve{F}_{\policyparams}, \gamma_{\state} \big\rangle - \mu \nonumber\\
        &\implies K^2 \Big( (\lambda \sigma)^2 + \sigma_{\gamma}^2 + 2\lambda \big(\big\langle \breve{F}_{\policyparams}, \gamma_{\state} \big\rangle - \mu\big) \Big) = \sigma^2 \Big(\lambda \sigma^2 + \big\langle \breve{F}_{\policyparams}, \gamma_{\state} \big\rangle - \mu\Big)^2 \nonumber\\
        &\implies K^2 \Big( \lambda^2 \sigma^2 + \sigma_{\gamma}^2 + 2\lambda \big(\big\langle \breve{F}_{\policyparams}, \gamma_{\state} \big\rangle - \mu\big) \Big) \nonumber\\
        &\qquad\qquad = \sigma^2 \left( \lambda^2 \sigma^4 + \left(\big\langle \breve{F}_{\policyparams}, \gamma_{\state} \big\rangle - \mu\right)^2 + 2 \lambda \sigma^2 \left(\big\langle \breve{F}_{\policyparams}, \gamma_{\state} \big\rangle - \mu\right) \right) \nonumber\\
        &\implies \lambda^2 \sigma^2 + 2 \lambda \left(\big\langle \breve{F}_{\policyparams}, \gamma_{\state} \big\rangle - \mu\right) + \frac{ K^2 \sigma_{\gamma}^2 - \sigma^2 \left(\big\langle \breve{F}_{\policyparams}, \gamma_{\state} \big\rangle - \mu\right)^2}{K^2 - \sigma^4} = 0. \label{eq:quadratic-lambda}
    \end{align}
    The discriminant of \cref{eq:quadratic-lambda}, a quadratic equation in $\lambda$, is 
    \begin{equation*}
        \Delta = 4 \Bigg( \left(\big\langle \breve{F}_{\policyparams}, \gamma_{\state} \big\rangle - \mu\right)^2 + \sigma^2  \frac{ \sigma^2 \big(\big\langle \breve{F}_{\policyparams}, \gamma_{\state} \big\rangle - \mu\big)^2 - K^2 \sigma_{\gamma}^2}{K^2 - \sigma^4} \Bigg),
    \end{equation*}
    must be nonnegative. Indeed, using the Cauchy-Schwarz inequality, we have respectively
    \begin{align*}
        K^2 - \sigma^4 &= \left(\big\langle  \breve{F}^{*}_{\adv} , \breve{F}_{\policyparams} \big\rangle - \mu^2 \right)^2 - \sigma^4 
        \leq \left(\big\Vert \breve{F}^{*}_{\adv} \big\Vert \big\Vert \breve{F}_{\policyparams} \big\Vert - \mu^2 \right)^2 - \sigma^4 = 0,
    \end{align*}
    and
    \begin{align*}
        \sigma^2 \left(\big\langle \breve{F}_{\policyparams}, \gamma_{\state} \big\rangle - \mu\right)^2 - K^2 \sigma_{\gamma}^2 &\leq \sigma^2 \left(\big\langle \breve{F}_{\policyparams}, \gamma_{\state} \big\rangle - \mu\right)^2 - \sigma^4 \sigma_{\gamma}^2 \leq 0.
    \end{align*}
    Therefore, the quadratic equation in \cref{eq:quadratic-lambda} has two roots, which only one is positive, more precisely
    \begin{equation*}
        \lambda^{*} = \frac{ -2\left(\big\langle \breve{F}_{\policyparams}, \gamma_{\state} \big\rangle - \mu\right) + \sqrt{\Delta} }{ 2\sigma^2 }.
    \end{equation*}
    If the uncertainty tolerance $\epsilon_{\state}$ is large enough and satisfies \cref{eq:larger-eps}, then $\mu + \frac{\gamma_{\state}(u) - 1}{b_0}$ solves the optimization problem $\qfunc_{\timeidx}(\state,\action;\policyparams)$. Indeed, it is admissible and for all $\lambda > 0$, using the Cauchy-Schwarz inequality, we have that
    \begin{align*}
        \left\langle \gamma_{\state} , \mu + \frac{\gamma_{\state}(u) - 1}{b_0} \right\rangle
        &= \mu + \sigma \sigma_{\gamma} \\
        &\geq \mu + \sigma \frac{\left\langle \gamma_{\state} , \left(\lambda \breve{F}_{\policyparams} + \gamma_{\state}\right) - \left(\lambda \mu + 1\right) \right\rangle}{\sqrt{ \left\Vert \lambda \breve{F}_{\policyparams} + \gamma_{\state} \right\Vert^2 - (\lambda\mu + 1)^2}} \\
        &= \left\langle \gamma_{\state} , \mu + \frac{\left(\lambda \breve{F}_{\policyparams} + \gamma_{\state}\right) - \left(\lambda \mu + 1\right)}{b_{\lambda}} \right\rangle.
    \end{align*}
    This concludes the proof.
\end{proof}

\subsection{Proof of Theorem \ref{thm:gradient-V}}
\label{sec:proof-gradient-V}

\begin{proof}
    Using the quantile representation and strong duality from \cref{thm:optimal-quantile-fn-mom}, we have
    \begin{equation*}
        \grad{\policyparams} \valuefunc_{\timeidx}(\state;\policyparams) = \grad{\policyparams} \min_{\lambda,\zeta,\eta>0} \max_{\breve{F}_{\adv}\in\breve{\FF}} L(\breve{F}_{\adv}, \lambda, \zeta, \eta; \policyparams).
    \end{equation*}
    We apply the Envelope theorem for saddle-point problems \citep{milgrom2002envelope}, which differs from standard results by considering arbitrary choice sets instead of convex ones.
    All conditions of the theorem hold, because the optimization problem is convex over the space of quantile functions:
    (i) $L(\breve{F}_{\adv},\lambda,\zeta,\eta;\policyparams)$ is absolutely continuous in $(\breve{F}_{\adv},\lambda,\zeta,\eta)$, because of the convexity and the fact that a distortion risk measure is monetary, and thus Lipschitz and absolutely continuous;
    (ii) $\grad{\policyparams} L(\breve{F}_{\adv},\lambda,\zeta,\eta;\policyparams)$ is continuous and bounded at each $(\breve{F}_{\adv},\lambda,\zeta,\eta)$, since $L$ is Lipschitz;
    (iii) there exists at least one saddle-point, as shown in \cref{thm:optimal-quantile-fn}; and
    (iv) $\{L(\breve{F}_{\adv},\lambda,\zeta,\eta;\policyparams)\}_{(\breve{F}_{\adv},\lambda,\zeta,\eta)}$ is equidifferentiable in $\policyparams$, i.e. its derivative wrt $\policyparams$ converges uniformly.
    This leads to
    \begin{align}
        \grad{\policyparams} \valuefunc_{\timeidx}(\state;\policyparams)  &= \grad{\policyparams} L(\breve{F}_{\adv}, \lambda, \zeta, \eta ; \policyparams)\Big|_{\left(\breve{F}_{\adv} = \breve{F}^{*}_{\adv}, \, \lambda=\lambda^{*}, \, \zeta=\zeta^{*}, \, \eta=\eta^{*}\right)} \nonumber\\
        \begin{split}
        & = \grad{\policyparams} \bigg( \big\langle \gamma_{\state} , \breve{F}_{\adv} \big\rangle
        - \lambda \Big( \big\Vert \breve{F}_{\adv} - \breve{F}_{\policyparams} \big\Vert^2 - \epsilon_{\state}^{2} \Big) 
        \\
        &\qquad\qquad - \zeta \Big( \big\Vert \breve{F}_{\adv} \big\Vert^2 - \big\Vert \breve{F}_{\policyparams} \big\Vert^2 \Big) - \eta \Big( \big\langle \breve{F}_{\adv},1\big\rangle - \mu \Big) \bigg) \bigg|_{\left(\breve{F}^{*}_{\adv}, \, \lambda^{*}, \, \zeta^{*}, \, \eta^{*}\right)}.
        \end{split} \label{eq:grad-lagrange}
    \end{align}
    For the first term of \cref{eq:grad-lagrange}, we have
    \begin{align}
        &\grad{\policyparams} \big\langle \gamma_{\state} , \breve{F}_{\adv} \big\rangle \Big|_{\left(\breve{F}^{*}_{\adv}, \, \lambda^{*}, \, \zeta^{*}, \, \eta^{*}\right)} \nonumber\\
        &\quad = \grad{\policyparams} \int_{0}^{1} \gamma_{\state}(u) \breve{F}_{\adv}(u|\state_{\timeidx},\policy^{\policyparams}(\state_{\timeidx})) \dee u \Big|_{\left(\breve{F}^{*}_{\adv}, \, \lambda^{*}, \, \zeta^{*}, \, \eta^{*}\right)} \nonumber\\
        &\quad = \grad{a} \int_{0}^{1} \gamma_{\state}(u) \breve{F}^{*}_{\adv}(u|\state_{\timeidx},a) \dee u \Big|_{a = \policy^{\policyparams}(\state_{\timeidx})} \grad{\policyparams} \policy^{\policyparams}(\state)  = \grad{a} \qfunc_{\timeidx}(\state, a; \policyparams)\Big|_{a = \policy^{\policyparams}(\state)} \; \grad{\policyparams} \policy^{\policyparams}(\state). \label{eq:grad-L-first-term}
    \end{align}
    For the second term of \cref{eq:grad-lagrange}, we get
    \begin{align}
        &-\grad{\policyparams} \lambda \Big( \big\Vert \breve{F}_{\adv} - \breve{F}_{\policyparams} \big\Vert^2 - \epsilon_{\state}^{2} \Big) \Big|_{\left(\breve{F}^{*}_{\adv}, \, \lambda^{*}, \, \zeta^{*}, \, \eta^{*}\right)} \nonumber\\
        &\quad = -\grad{\policyparams} \lambda \bigg( \int_{0}^{1} \Big(\breve{F}_{\adv}(u|\state_{\timeidx},\policy^{\policyparams}(\state_{\timeidx})) - \breve{F}_{\policyparams}(u|\state_{\timeidx},\policy^{\policyparams}(\state_{\timeidx}))\Big)^{2} \dee u - \epsilon_{\state}^{2} \bigg) \bigg|_{\left(\breve{F}^{*}_{\adv}, \, \lambda^{*}, \, \zeta^{*}, \, \eta^{*}\right)} \nonumber\\
        &\quad = -2 \lambda^{*} \int_{0}^{1} \Bigg( \breve{F}^{*}_{\adv}(u|\state_{\timeidx},\policy^{\policyparams}(\state_{\timeidx})) - \breve{F}_{\policyparams}(u|\state_{\timeidx},\policy^{\policyparams}(\state_{\timeidx}))\Bigg) \nonumber\\
        &\qquad\qquad\qquad \times \grad{\policyparams} \Big(\breve{F}_{\adv}(u|\state_{\timeidx},\policy^{\policyparams}(\state_{\timeidx})) - \breve{F}_{\policyparams}(u|\state_{\timeidx},\policy^{\policyparams}(\state_{\timeidx}))\Big) \dee u \Big|_{\left(\breve{F}^{*}_{\adv}, \, \lambda^{*}, \, \zeta^{*}, \, \eta^{*}\right)} \nonumber\\
        &\quad = -2 \lambda^{*} \left( \grad{\policyparams} \policy^{\policyparams}(\state) \right) \int_{0}^{1} \Big( \breve{F}^{*}_{\adv}(u|\state_{\timeidx},\policy^{\policyparams}(\state_{\timeidx})) - \breve{F}_{\policyparams}(u|\state_{\timeidx},\policy^{\policyparams}(\state_{\timeidx}))\Big) \nonumber\\
        &\qquad\qquad\qquad \times \grad{a} \Big( \breve{F}^{*}_{\adv}(u|\state_{\timeidx},a) - \breve{F}_{\policyparams}(u|\state_{\timeidx},a) \Big)\Big|_{a=\policy^{\policyparams}(\state_{\timeidx})} \dee u \nonumber\\
        &\quad = -2 \lambda^{*} \left( \grad{\policyparams} \policy^{\policyparams}(\state) \right) \int_{0}^{1} \left( \frac{2\lambda^{*} \breve{F}_{\policyparams}(u|\state_{\timeidx},\policy^{\policyparams}(\state_{\timeidx})) + \gamma_{\state}(u) - \eta^{*}}{2 (\lambda^{*} + \zeta^{*})} - \breve{F}_{\policyparams}(u|\state_{\timeidx},\policy^{\policyparams}(\state_{\timeidx})) \right) \nonumber\\
        &\qquad\qquad\qquad \times \left( \frac{2\lambda^{*}}{2 (\lambda^{*} + \zeta^{*})} - 1 \right) \grad{a} \breve{F}_{\policyparams}(u|\state_{\timeidx},a) \Big|_{a=\policy^{\policyparams}(\state_{\timeidx})} \dee u \nonumber\\
        &\quad = \frac{4 \lambda^{*} \zeta^{*} \left( \grad{\policyparams} \policy^{\policyparams}(\state) \right)}{4 (\lambda^{*} + \zeta^{*})^2}  \int_{0}^{1} \left( \gamma_{\state}(u) - \eta^{*} - 2\zeta^{*} \breve{F}_{\policyparams}(u|\state_{\timeidx},\policy^{\policyparams}(\state_{\timeidx})) \right) \grad{a} \breve{F}_{\policyparams}(u|\state_{\timeidx},a) \Big|_{a=\policy^{\policyparams}(\state_{\timeidx})} \dee u. \label{eq:grad-L-second-term}
    \end{align}
    For the third term of \cref{eq:grad-lagrange}, we have
    \begin{align}
        &-\grad{\policyparams} \zeta \Big( \big\Vert \breve{F}_{\adv} \big\Vert^2 - \big\Vert \breve{F}_{\policyparams} \big\Vert^2 \Big) \Big|_{\left(\breve{F}^{*}_{\adv}, \, \lambda^{*}, \, \zeta^{*}, \, \eta^{*}\right)} \nonumber\\
        &\quad = -\grad{\policyparams} \zeta \int_{0}^{1} \Big(\breve{F}_{\adv}(u|\state_{\timeidx},\policy^{\policyparams}(\state_{\timeidx}))\Big)^{2} - \Big(\breve{F}_{\policyparams}(u|\state_{\timeidx},\policy^{\policyparams}(\state_{\timeidx}))\Big)^{2} \dee u \Big|_{\left(\breve{F}^{*}_{\adv}, \, \lambda^{*}, \, \zeta^{*}, \, \eta^{*}\right)} \nonumber\\
        \begin{split}
         &\quad = 2 \zeta^{*} \left( \grad{\policyparams} \policy^{\policyparams}(\state) \right) \int_{0}^{1} \breve{F}_{\policyparams}(u|\state_{\timeidx},\policy^{\policyparams}(\state_{\timeidx})) \grad{a} \breve{F}_{\policyparams}(u|\state_{\timeidx},a) \Big|_{a=\policy^{\policyparams}(\state_{\timeidx})} \dee u \\
        &\qquad\quad -\frac{4 \lambda^{*} \zeta^{*} \left( \grad{\policyparams} \policy^{\policyparams}(\state) \right)}{4 (\lambda^{*} + \zeta^{*})^2} \int_{0}^{1} \Big(2\lambda^{*} \breve{F}_{\policyparams}(u|\state_{\timeidx},\policy^{\policyparams}(\state_{\timeidx})) + \gamma_{\state}(u) - \eta^{*}\Big) \grad{a} \breve{F}_{\policyparams}(u|\state_{\timeidx},a) \Big|_{a=\policy^{\policyparams}(\state_{\timeidx})} \dee u.   
        \end{split}\label{eq:grad-L-third-term}
    \end{align}
    For the fourth term of \cref{eq:grad-lagrange}, we have
    \begin{align}
        &- \grad{\policyparams}\eta \Big( \big\langle \breve{F}_{\adv},1\big\rangle - \mu \Big) \Big|_{\left(\breve{F}^{*}_{\adv}, \, \lambda^{*}, \, \zeta^{*}, \, \eta^{*}\right)} \nonumber\\
        &\quad = -\grad{\policyparams} \eta \int_{0}^{1} \breve{F}_{\adv}(u|\state_{\timeidx},\policy^{\policyparams}(\state_{\timeidx})) - \breve{F}_{\policyparams}(u|\state_{\timeidx},\policy^{\policyparams}(\state_{\timeidx})) \dee u \Big|_{\left(\breve{F}^{*}_{\adv}, \, \lambda^{*}, \, \zeta^{*}, \, \eta^{*}\right)} \nonumber\\
        &\quad = - \eta^{*} \left( \grad{\policyparams} \policy^{\policyparams}(\state) \right) \int_{0}^{1} \grad{a} \Big( \breve{F}^{*}_{\adv}(u|\state_{\timeidx},a) - \breve{F}_{\policyparams}(u|\state_{\timeidx},a) \Big)\Big|_{a=\policy^{\policyparams}(\state_{\timeidx})} \dee u \nonumber\\
        &\quad = - \eta^{*} \left( \grad{\policyparams} \policy^{\policyparams}(\state) \right) \int_{0}^{1} \left( \frac{2\lambda^{*}}{2 (\lambda^{*} + \zeta^{*})} - 1 \right) \grad{a} \breve{F}_{\policyparams}(u|\state_{\timeidx},a) \Big|_{a=\policy^{\policyparams}(\state_{\timeidx})} \dee u \nonumber\\
        &\quad = \frac{2\eta^{*}\zeta^{*}}{2(\lambda^{*} + \zeta^{*})} \left( \grad{\policyparams} \policy^{\policyparams}(\state) \right) \int_{0}^{1} \grad{a} \breve{F}_{\policyparams}(u|\state_{\timeidx},a) \Big|_{a=\policy^{\policyparams}(\state_{\timeidx})} \dee u. \label{eq:grad-L-fourth-term}
    \end{align}
    Combining all terms in \cref{eq:grad-L-first-term,eq:grad-L-second-term,eq:grad-L-third-term,eq:grad-L-fourth-term} together, we get
    \begin{align*}
        &\grad{\policyparams} \valuefunc_{\timeidx}(\state;\policyparams) \\
        &\quad = \grad{a} \qfunc_{\timeidx}(\state, a; \policyparams)\Big|_{a = \policy^{\policyparams}(\state)} \; \grad{\policyparams} \policy^{\policyparams}(\state) + \frac{2\eta^{*}\zeta^{*}}{2(\lambda^{*} + \zeta^{*})} \left( \grad{\policyparams} \policy^{\policyparams}(\state) \right) \int_{0}^{1} \grad{a} \breve{F}_{\policyparams}(u|\state_{\timeidx},a) \Big|_{a=\policy^{\policyparams}(\state_{\timeidx})} \dee u \\
        &\qquad\quad + \left(2 \zeta^{*} - \frac{4 \lambda^{*} \zeta^{*}}{2 (\lambda^{*} + \zeta^{*})} \right)  \left( \grad{\policyparams} \policy^{\policyparams}(\state) \right) \int_{0}^{1} \Big(\breve{F}_{\policyparams}(u|\state_{\timeidx},\policy^{\policyparams}(\state_{\timeidx}))\Big) \grad{a} \breve{F}_{\policyparams}(u|\state_{\timeidx},a) \Big|_{a=\policy^{\policyparams}(\state_{\timeidx})} \dee u \\
        &\quad = \grad{\policyparams} \policy^{\policyparams}(\state) \Bigg( \grad{a} \qfunc_{\timeidx}(\state, a; \policyparams)\Big|_{a = \policy^{\policyparams}(\state)} \\
        &\qquad\quad + \frac{2 \zeta^{*}}{2 (\lambda^{*} + \zeta^{*})}  \int_{0}^{1} \Big(2 \zeta^{*} \breve{F}_{\policyparams}(u|\state_{\timeidx},\policy^{\policyparams}(\state_{\timeidx})) + \eta^{*}\Big) \grad{a} \breve{F}_{\policyparams}(u|\state_{\timeidx},a) \Big|_{a=\policy^{\policyparams}(\state_{\timeidx})} \dee u \Bigg).
    \end{align*}
    Using the notation in \cref{thm:optimal-quantile-fn-mom} with $\lambda,b_{\lambda},a_{\lambda}$, especially \cref{eq:a-b-lambda}, and interpreting the integral as an expectation over a uniform random variable, we get the gradient of the value function and, hence, the Q-function:
    \begin{align*}
        &\grad{\policyparams} \valuefunc_{\timeidx}(\state;\policyparams) \\
        &\quad = \grad{\policyparams} \policy^{\policyparams}(\state) \Bigg( \grad{a} \qfunc_{\timeidx}(\state, a; \policyparams)\Big|_{a = \policy^{\policyparams}(\state)} \\
        &\qquad\quad + \frac{b_{\lambda^{*}} - \lambda^{*}}{b_{\lambda^{*}}} \int_{0}^{1} \Big((b_{\lambda^{*}} - \lambda^{*}) (\breve{F}_{\policyparams}(u|\state_{\timeidx},\policy^{\policyparams}(\state_{\timeidx})) - \mu) + 1\Big) \grad{a} \breve{F}_{\policyparams}(u|\state_{\timeidx},a) \Big|_{a=\policy^{\policyparams}(\state_{\timeidx})} \dee u \Bigg) \\
        &\quad = \grad{\policyparams} \policy^{\policyparams}(\state) \Bigg( \grad{a} \qfunc_{\timeidx}(\state, a; \policyparams)\Big|_{a = \policy^{\policyparams}(\state)} \\
        &\qquad\quad + \frac{b_{\lambda^{*}} - \lambda^{*}}{b_{\lambda^{*}}} \EE_{\timeidx,\state} \left[ \Big((b_{\lambda^{*}} - \lambda^{*}) (\rv_{\timeidx}^{\policyparams} - \mu) + 1\Big) \; \grad{a} \breve{F}_{\policyparams}(x|\state,a) \Big|_{(x,a)=(\rv_{\timeidx}^{\policyparams},\policy^{\policyparams}(\state))} \right] \Bigg) \\
        &\quad = \grad{\policyparams} \policy^{\policyparams}(\state) \Bigg( \grad{a} \qfunc_{\timeidx}(\state, a; \policyparams)\Big|_{a = \policy^{\policyparams}(\state)} \\
        &\qquad\quad - \frac{b_{\lambda^{*}} - \lambda^{*}}{b_{\lambda^{*}}} \EE_{\timeidx,\state} \left[ \Big((b_{\lambda^{*}} - \lambda^{*}) (\rv_{\timeidx}^{\policyparams} - \mu) + 1\Big) \; \frac{ \grad{a} F_{\policyparams}(x|\state,a)}{f_{\policyparams}(x|\state,a)} \bigg|_{(x,a)=(\rv_{\timeidx}^{\policyparams},\policy^{\policyparams}(\state))} \right] \Bigg) \\
        &\quad = \grad{\policyparams} \policy^{\policyparams}(\state) \Bigg( \grad{a} \qfunc_{\timeidx}(\state, a; \policyparams)\Big|_{a = \policy^{\policyparams}(\state)} \\
        &\qquad\quad - \frac{b_{\lambda^{*}} - \lambda^{*}}{b_{\lambda^{*}}} \EE_{\timeidx,\state} \left[ \Big((b_{\lambda^{*}} - \lambda^{*}) (\rv_{\timeidx}^{\policyparams} - \mu) + 1\Big) \; \frac{ \grad{a} F_{\policyparams}(x|\state,a)}{\grad{x} F_{\policyparams}(x|\state,a)} \bigg|_{(x,a)=(\rv_{\timeidx}^{\policyparams},\policy^{\policyparams}(\state))} \right] \Bigg).
    \end{align*}
    Here, we write the gradient of a quantile function by instead considering the gradient of the \CDF{}. Indeed, for any \CDF{} $F_{\policyparams}$, using the fact that $F_{\policyparams}(\breve{F}_{\policyparams}(u)) = u$ and the chain rule to expand the gradient in terms of the partial derivatives, we have
    \begin{align*}
        &\grad{\policyparams} F_{\policyparams}(\breve{F}_{\policyparams}(u)) = \grad{\policyparams} u \\
        &\iff \grad{x} F_{\policyparams}(x)\Big|_{x=\breve{F}_{\policyparams}(u)} \grad{\policyparams} \breve{F}_{\policyparams}(u) + \grad{\policyparams} F_{\policyparams}(x)\Big|_{x=\breve{F}_{\policyparams}(u)} = 0 \\
        &\iff \grad{\policyparams} \breve{F}_{\policyparams}(u) = - \frac{\grad{\policyparams} F_{\policyparams}(x)\Big|_{x=\breve{F}_{\policyparams}(u)}}{f_{\policyparams}(\breve{F}_{\policyparams}(u))}.
    \end{align*}
\end{proof}

\end{document}